\RequirePackage{fix-cm}
\documentclass[smallextended,natbib]{svjour3}
\smartqed
\usepackage{amsmath,amssymb}
\usepackage{graphicx,booktabs,xcolor,array,tabularx}
\usepackage[expansion=false]{microtype}
\usepackage{enumitem}
\usepackage[hidelinks]{hyperref}
\setlist{leftmargin=1.4em,itemsep=2pt,topsep=3pt}
\newenvironment{proof}[1][\proofname]{\par\addvspace{\medskipamount}\noindent{\itshape #1.}\enspace\ignorespaces}{\qed\par\addvspace{\medskipamount}}
\makeatletter\def\makeheadbox{}\makeatother

\newcommand{\Vp}{V_p^{\max}}
\newcommand{\Ve}{V_e^{\max}}
\newcommand{\Om}{\Omega_p^{\max}}
\newcommand{\kmax}{\kappa_p^{\max}}
\newcommand{\x}{\mathbf{x}}

\newcommand{\UP}{\ensuremath{\mathrm{UP}}}
\newcommand{\NUP}{\ensuremath{\mathrm{NUP}}}
\newcommand{\BUP}{\ensuremath{\mathrm{BUP}}}
\newcommand{\Hset}{\mathcal{H}}
\newcommand{\Eset}{\mathcal{E}}
\newcommand{\De}{\mathcal{D}_e}

\newcommand{\Bbar}{\mathcal{B}}
\newcommand{\sgn}{\operatorname{sgn}}

\begin{document}

\title{Minimum-Time Escape from the Conic Field of View of a Reeds--Shepp Pursuer: Classification and Game Value for a Slower or Equally Fast Evader}
\titlerunning{Escape from the conic field of view of a Reeds--Shepp pursuer}
\author{Ubaldo Ruiz}
\authorrunning{U. Ruiz}
\institute{U. Ruiz \at
  Centro de Investigaci\'on Cient\'ifica y de Educaci\'on Superior de Ensenada (CICESE), Ensenada, Baja California, Mexico\\
  \email{uruiz@cicese.edu.mx}}
\date{Preprint, October 2026}

\maketitle

\begin{abstract}
We study minimum-time escape of an omnidirectional evader from the forward conic field of view of a Reeds--Shepp pursuer with minimum turning radius $r$. For speed ratio $\nu$ and cone half-angle $0<\alpha<\pi/2$, we prove a global escape/holding classification: straight reversal holds the whole cone for $\nu\le\sin\alpha$, escape is confined to a barrier-bounded dart for $\sin\alpha<\nu\le1$, and escape is possible from every state for $\nu>1$. For the slower or equally fast evader, $\sin\alpha<\nu\le1$, we construct the solution with Isaacs' characteristics and verify it in feedback form, using a Filippov pursuer relay and a one-sided evader branch. The finite escape time is the first root of one scalar equation, and the holding region is bounded by the nonusable parts of the rays and a barrier that closes with a convex corner on the front axis. Sample-and-hold counter-controls show that this escape time is the value of the game in nonanticipative strategies at every state except the corner. There the two strategy values differ: the pursuer's is infinite, whereas an evader that chooses which barrier branch to cross from the initial germ of the pursuer's control escapes before the barrier time $\tau_I$, which is the evader's exact strategy value. An independent level-set solver, refined on fixed comparison points, agrees with the closed-form escape time and reproduces the holding region at the tested parameters.
\end{abstract}

\section{Introduction}\label{sec:intro}
A directional sensor attached to a moving vehicle observes only part of the plane. Escaping surveillance therefore requires leaving its field of view, rather than outrunning the observer. We study minimum-time escape from an unbounded forward cone of half-angle $0<\alpha<\pi/2$, with an omnidirectional evader and a Reeds--Shepp pursuer that can drive forward or in reverse but cannot rotate in place \citep{reeds1990}. The evader minimizes the exit time and the pursuer maximizes it.

The closest forward-only model is the conic surveillance--evasion game of \citet{lewin1979conic}. Reverse motion changes its geometry: backing away reduces the evader's bearing without turning and can preserve observation indefinitely.

Our analysis uses Isaacs' characteristics and singular surfaces \citep{isaacs1965,lewin1994}. Related surveillance games include circular detection regions \citep{dobbie1966,lewin1975}, an isotropic rocket \citep{lewin1989rocket}, two identical cars \citep{greenfeld1987,ruiz2024dubins}, an agile pursuer \citep{braun2025}, and a fixed-course target \citep{weintraub2023}. Differential-drive capture and surveillance games are developed in several works \citep{ruiz2013,ruiz2016,ruiz2022,saavedra2024}, including escape from a conic field of view \citep{ruiz2023fov}. Their velocity set is a diamond, allowing rotation in place; the Reeds--Shepp set is a bow tie, allowing full speed and full turn rate simultaneously. The Hamiltonians and characteristics therefore differ.

The contributions are the game-of-kind classification for all speed ratios $\nu=\Ve/\Vp$ and the solution for $\nu\le1$. The solution is constructed from the usable boundary and the primary characteristics, verified in feedback form with the pursuer's relay understood in the sense of Filippov, and identified with the strategy values by sample-and-hold counter-controls and an explicit evader strategy at the corner. An independent level-set solver, refined on fixed comparison points, agrees with the closed-form escape time. The main results are the following.
\begin{enumerate}[label=(\roman*)]
\item For $\nu\le\sin\alpha$, straight reversal holds the whole cone (Proposition~\ref{prop:kind}).
\item For $\sin\alpha<\nu\le1$, the escape region is a dart bounded by the usable parts of the rays and a closed barrier with a corner $c$ on the front axis, and the escape time $W$ is the first root of one scalar equation (Proposition~\ref{prop:Wform} and Theorem~\ref{thm:slow}).
\item $W$ is the value of the game in nonanticipative strategies at every state except $c$; at $c$, $V_P(c)=+\infty$ while $V_E(c)=\tau_I$ (Theorem~\ref{thm:value}).
\item For $\nu>1$, escape is possible from every state (Proposition~\ref{prop:kind}).
\end{enumerate}
The holding region $\Hset$ is bounded by the barrier and the nonusable parts of the rays; it is defined geometrically and then shown to be held by the pursuer. At its corner the pursuer's holding capability depends on the information pattern: it holds when the pursuer's strategy reacts to the evader, but not when the evader's strategy reacts to the pursuer.

For $\nu>1$ the minimum escape time has a different structure: the characteristic analysis indicates a universal line on the front axis, along which the evader chases the reversing pursuer, and a barrier that ends at a cusp. Its verification requires other arguments and is the subject of ongoing work; here the case $\nu>1$ only completes the classification.

Sections~\ref{sec:problem}--\ref{sec:kind} formulate the game, derive its characteristics and solve the game of kind. Section~\ref{sec:slow} constructs and verifies the synthesis for the slower or equally fast evader, Section~\ref{sec:num} compares it with an independent numerical solution, and Section~\ref{sec:lewin} compares the geometry with the forward-only game. Appendix~\ref{app:relay} proves Lemma~\ref{lem:relay}, and Appendix~\ref{app:value} proves Theorem~\ref{thm:value}.

\section{Problem formulation}\label{sec:problem}
We state the game in a frame attached to the pursuer and reduce it to the two parameters $\nu$ and $\alpha$.

The pursuer configuration is $\mathbf{q}_p=(X_p,Y_p,\theta_p)^\top$, with position $(X_p,Y_p)$ and heading $\theta_p$ in a global frame $(X,Y)$, and kinematics
\begin{equation}\label{eq:pursuer}
\dot X_p=v_p\cos\theta_p,\quad \dot Y_p=v_p\sin\theta_p,\quad \dot\theta_p=\omega_p=v_p\kappa_p,
\end{equation}
$|v_p|\le\Vp$, $|\kappa_p|\le\kmax$, so $|\omega_p|\le\Om=\Vp\kmax$ and $v_p=0$ implies $\omega_p=0$: the admissible pairs form the bow tie $\{|v_p|\le\Vp,\ |\omega_p|\le\kmax|v_p|\}$. The evader is omnidirectional, $\dot{\mathbf p}_e=\mathbf u_e$, $\|\mathbf u_e\|\le\Ve$, and $\nu=\Ve/\Vp\in(0,\infty)$; the classification covers every $\nu$, and the synthesis treats a slower ($\nu<1$) and an equally fast ($\nu=1$) evader together. We write $r=\Vp/\Om=1/\kmax$ for the minimum turning radius. Here $\mathbf p_e=(X_e,Y_e)^\top$ is the evader position and $\mathbf u_e=(u_X,u_Y)^\top$ its velocity, both in the global frame.

\paragraph{Reduced frame} Only the position of the evader relative to the pursuer and to its heading matters, so we use a frame attached to the pursuer, with origin at $(X_p,Y_p)$, $y$-axis along the heading $(\cos\theta_p,\sin\theta_p)^\top$ and $x$-axis along $(\sin\theta_p,-\cos\theta_p)^\top$, which points to the pursuer's right (Fig.~\ref{fig:geom}). The reduced state $\x=(x,y)^\top$ is the evader position in this frame,
\begin{equation}\label{eq:transf}
\begin{aligned}
x&=(X_e-X_p)\sin\theta_p-(Y_e-Y_p)\cos\theta_p,\\
y&=(X_e-X_p)\cos\theta_p+(Y_e-Y_p)\sin\theta_p,
\end{aligned}
\end{equation}
and the evader velocity expressed in the same frame is
\begin{equation}\label{eq:uframe}
\begin{aligned}u_x&=u_X\sin\theta_p-u_Y\cos\theta_p,\\ u_y&=u_X\cos\theta_p+u_Y\sin\theta_p,\end{aligned}
\end{equation}
so that $\|\mathbf u\|=\|\mathbf u_e\|\le\Ve$, with $\mathbf u=(u_x,u_y)^\top$. Differentiating~\eqref{eq:transf} along~\eqref{eq:pursuer} and $\dot{\mathbf p}_e=\mathbf u_e$, the terms in $\omega_p$ give $\omega_py$ and $-\omega_px$, the terms in $v_p$ cancel in $\dot x$ and give $-v_p$ in $\dot y$, so
\begin{equation}\label{eq:reduced}
\dot x=u_x+\omega_p y,\qquad \dot y=u_y-v_p-\omega_p x .
\end{equation}
The pursuer's position and heading do not appear in~\eqref{eq:reduced}, so the game is played in the reduced plane.

\begin{figure}[!htbp]
\centering
\includegraphics[width=\textwidth]{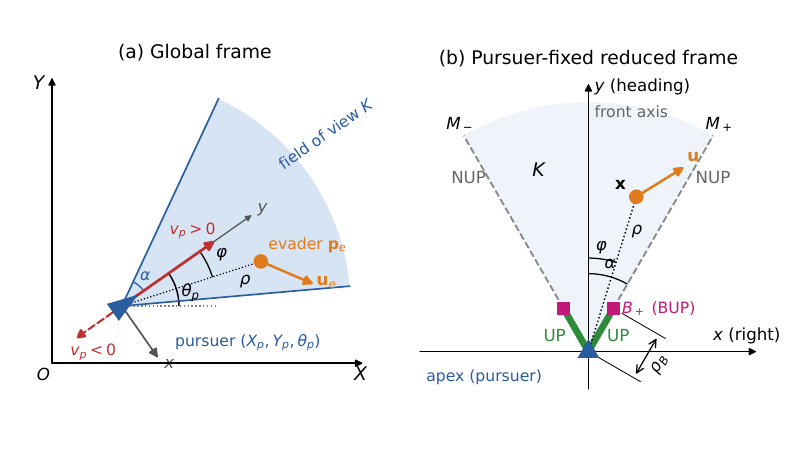}
\caption{Conic field of view and coordinates. (a) Global frame: pursuer heading $\theta_p$, evader position $\mathbf p_e$, range $\rho$ and clockwise bearing $\varphi$. The cone is drawn truncated but has unbounded range. (b) Pursuer-fixed frame: $x$ points right and $y$ along the heading. The usable segments UP (green), of length $\rho_B=r(\nu-\sin\alpha)$, end at $B_\pm$ (magenta); the remaining rays are NUP (gray dashed). Both panels use $\alpha=30^\circ$; (b) uses $\nu=0.9$.}

\label{fig:geom}
\end{figure}

\paragraph{Polar form} Let $\rho=\|\x\|$ and $\varphi=\operatorname{atan2}(x,y)\in(-\pi,\pi]$ be the bearing, measured clockwise from the heading. With $\mathbf e_r=(\sin\varphi,\cos\varphi)^\top$, $\mathbf e_\varphi=(\cos\varphi,-\sin\varphi)^\top$, $u_r=\mathbf u^\top\mathbf e_r$ and $u_\varphi=\mathbf u^\top\mathbf e_\varphi$,
\begin{equation}\label{eq:polar}
\dot\rho=u_r-v_p\cos\varphi,\qquad \rho\dot\varphi=u_\varphi+\omega_p\rho+v_p\sin\varphi .
\end{equation}
Turning moves the evader angularly at a rate proportional to its distance. Reversing ($v_p<0$) pushes the evader away and decreases $|\varphi|$; driving forward does the opposite.

\paragraph{Surveillance cone, escape condition and payoff} The field of view is the closed cone $K=\{\x:|\varphi|\le\alpha\}\cup\{\mathbf 0\}$ with $0<\alpha<\pi/2$, so $K$ is convex. The game terminates at the first strict exit,
\begin{equation}\label{eq:tf}
t_f=\inf\{t\ge0:\x(t)\notin K\},
\end{equation}
with $t_f=+\infty$ if this never happens, and the payoff is $J=t_f$. Contact with a ray does not end the game. Optimal play is described by the values of the game in nonanticipative strategies \citep{bardi1997,elliott1972}: a strategy of one player maps the opponent's control function to its own. Controls are measurable and identified when they agree almost everywhere. A map $A$ from pursuer controls to evader controls is nonanticipative if, for every $s\ge0$, $\mathbf a=\mathbf b$ a.e.\ on $[0,s]$ implies $A[\mathbf a]=A[\mathbf b]$ a.e.\ on $[0,s]$; pursuer strategies are defined in the same way. No observation delay is imposed, so a player may respond to the opponent's current control. With $\mathbf u(\cdot)$ the evader's control, $\mathbf a(\cdot)=(v_p(\cdot),\omega_p(\cdot))$ the pursuer's control, and $\mathcal S_E$, $\mathcal S_P$ the nonanticipative strategies of the evader and of the pursuer, the two strategy values are
\begin{align*}
V_E(\x)&=\inf_{A\in\mathcal S_E}\sup_{\mathbf a}J\bigl(A[\mathbf a],\mathbf a\bigr),\\
V_P(\x)&=\sup_{B\in\mathcal S_P}\inf_{\mathbf u}J\bigl(\mathbf u,B[\mathbf u]\bigr),
\end{align*}
in which the evader or the pursuer, respectively, reacts to the opponent's control. Each expression pairs one strategy with an open-loop control of the opponent; no consistency of two strategies played against each other is assumed. We denote by $W(\x)$ the escape time given by the synthesis constructed below. For $\sin\alpha<\nu\le1$, Theorem~\ref{thm:slow} verifies it in feedback form, as two unilateral guarantees: a pursuer relay guarantees at least $W$ against every measurable evader control, and an evader feedback guarantees at most $W$ against every measurable pursuer control. Both define nonanticipative strategies, so $V_E\le W\le V_P$ (Corollary~\ref{cor:sandwich}), and Theorem~\ref{thm:value} shows that $V_E=V_P=W$ at every state except one corner point. We call $W$ the \emph{value} where Theorem~\ref{thm:value} identifies both strategy values. Informally, the \emph{holding region} $\Hset$ is the set of states from which the pursuer can keep the evader in $K$ forever, and $K\setminus\Hset$ is the \emph{escape region}. For $\sin\alpha<\nu\le1$ we first define $\Hset$ geometrically (Section~\ref{sec:slow}), set $W=+\infty$ on it, and only then establish its game-theoretic meaning: the relay keeps every state of $\Hset$ in $K$ against every evader control (Theorem~\ref{thm:slow}), and $\Hset=\{V_P=+\infty\}$ (Theorem~\ref{thm:value}). By the strict-exit convention, $W=0$ on the usable part of the rays and at a usable apex, whereas for $\nu\le1$ the barrier belongs to $\Hset$.

\paragraph{Feedback controls and Filippov solutions} Some feedback laws used below are discontinuous, e.g., the relay $\omega_p=-\Om\sgn x$. Trajectories under a discontinuous feedback are understood in the sense of \citet{filippov1988}: absolutely continuous solutions of the differential inclusion obtained by taking, at each point, the closed convex hull of the limits of the right-hand side. On the line $x=0$ the relay then produces sliding with the equivalent turn rate $\omega_p=-u_x/y$ whenever $|u_x|\le\Om y$, which is determined by the evader's current control; this is the sense in which the pursuer reacts to the evader's current control in Theorem~\ref{thm:slow}. The Filippov convention is used for the pursuer's relay for each prescribed measurable evader control. It is not applied to the evader's feedback: on the evader's axis dispersal line the evader selects one of the two one-sided branches and applies it as its actual control, and the trajectory enters the chosen half-plane. We do not convexify the two evader branches into a motion along the axis; that convexification would allow trajectories, sliding up the axis, whose escape time exceeds $W$. Conditions under which discontinuous feedbacks can be interpreted as nonanticipative strategies are discussed by \citet{vinter2004}.

\paragraph{Scale invariance} $K$ is a cone, so lengths can be measured in units of $r$ and times in units of $1/\Om$. Substituting $\x=r\hat\x$ and $t=\hat t/\Om$ into~\eqref{eq:reduced} gives the same game with $\Vp=\Om=r=1$ and $\Ve=\nu$. Hence
\begin{equation}\label{eq:scale}
W(\x;\Vp,\Om,\Ve,\alpha)=\frac{1}{\Om}\,\widehat W\!\left(\frac{\x}{r};\nu,\alpha\right),
\end{equation}
and the whole game depends on the two parameters $(\nu,\alpha)$.

\begin{lemma}[Relay solutions]\label{lem:relay}
In normalized units, let $\mathbf u(\cdot)$ be a measurable evader control and let the pursuer play $v_p=-1$, $\omega_p=-\sgn x$. From every initial state with $y>0$, the Filippov inclusion
\[
\begin{aligned}
&\dot\x\in\mathbf u(t)+(0,1)^\top+F(\x),\\
&F(x,y)=\{(-\varsigma y,\varsigma x)^\top:\ \varsigma\in\operatorname{Sign}(x)\},
\end{aligned}
\]
with $\operatorname{Sign}(0)=[-1,1]$, has exactly one solution while $y>0$. Along it, $\omega_p(t)=-\sgn x(t)$ where $x(t)\neq0$ and $\omega_p(t)=-u_x(t)/y(t)$ for almost every $t$ with $x(t)=0$; this turn rate is measurable, $(v_p,\omega_p)$ lies in the bow tie, and the solution is the ordinary solution of~\eqref{eq:reduced} under $(\mathbf u,v_p,\omega_p)$. The map $\mathbf u(\cdot)\mapsto(v_p,\omega_p)(\cdot)$ is nonanticipative.
\end{lemma}
Lemma~\ref{lem:relay} makes the pursuer's relay a well-defined nonanticipative strategy, which Theorem~\ref{thm:slow} and Corollary~\ref{cor:sandwich} use; the proof, a uniqueness argument for the Filippov inclusion, is in Appendix~\ref{app:relay}. In normalized units every relay solution has $\dot y=u_y+1+|x|$, so for $\nu\le1$, $\dot y\ge1-\nu\ge0$: a solution starting in $K\setminus\{\mathbf 0\}$, where $y>0$, keeps $y>0$, and Lemma~\ref{lem:relay} applies for all time.

\section{Target manifold and usable part}\label{sec:target}
We identify the part of the boundary of the cone through which the evader can force its exit; it supplies the terminal data of the characteristics.

The terminal manifold consists of the right ray $M_+=\{\rho(\sin\alpha,\cos\alpha)^\top:\rho>0\}$, its mirror image $M_-$, and the apex. By the lateral symmetry $x\mapsto-x$ it suffices to consider $M_+$, where the evader escapes by making $\dot\varphi>0$. By~\eqref{eq:polar},
\begin{equation}\label{eq:saddle_ray}
\max_{\|\mathbf u\|\le\Ve}\ \min_{(v_p,\omega_p)}\rho\dot\varphi=\Ve-\Om\rho-\Vp\sin\alpha,
\end{equation}
attained by $\mathbf u=\Ve\mathbf e_\varphi$ and by $(v_p,\omega_p)=(-\Vp,-\Om)$, i.e., full reverse with $\kappa_p=+\kmax$, an admissible Reeds--Shepp control. Escape can be enforced where~\eqref{eq:saddle_ray} is positive, so
\begin{align}
\UP&=\{\rho(\sin\alpha,\cos\alpha)^\top:\ 0<\rho<\rho_B\},\label{eq:UP}\\
\rho_B&=\frac{\Ve-\Vp\sin\alpha}{\Om}=r(\nu-\sin\alpha),\label{eq:rhoB}
\end{align}
with $\BUP=\{\rho=\rho_B\}$ and $\NUP=\{\rho>\rho_B\}$. The usable part is nonempty if and only if $\nu>\sin\alpha$; for $\nu\le\sin\alpha$ the usable part and the BUP are empty and each ray is entirely nonusable. For $\nu>\sin\alpha$ the BUP consists of the two points $B_\pm=\rho_B(\pm\sin\alpha,\cos\alpha)^\top$, at distance $\rho_B$ from the apex, and we write $\UP_\pm$ for the usable parts of $M_\pm$, the segments from the apex to $B_\pm$. The usable part lies next to the apex, where the angular sweep $\Om\rho$ produced by turning is small.

\paragraph{No sliding on the nonusable part} On the rays the pursuer's radial and angular objectives agree: the control $(-\Vp,-\Om)$ that minimizes $\rho\dot\varphi$ also maximizes $\dot\rho=u_r-v_p\cos\varphi$, giving $u_r+\Vp\cos\varphi$. On the $\NUP$ it gives $\rho\dot\varphi\le\Ve-\Om\rho-\Vp\sin\alpha<0$ against every evader control. The evader therefore cannot force grazing against this inward control. The constructed synthesis needs no boundary sliding arc; this does not exclude grazing under other control pairs.

\paragraph{The apex} The apex is the corner of $K$, and at $\x=\mathbf 0$ the turning terms of~\eqref{eq:reduced} vanish: $\dot x=u_x$ and $\dot y=u_y-v_p$. Let $\ell=x\cos\alpha-y\sin\alpha$, so that $K\subset\{\ell\le0\}$. With $\mathbf u=\Ve(\cos\alpha,-\sin\alpha)$, $\dot\ell=\Ve+v_p\sin\alpha+\omega_p(x\sin\alpha+y\cos\alpha)$, which equals at least $\Ve-\Vp\sin\alpha$ at the apex. Hence, for $\nu>\sin\alpha$, an evader at the apex leaves $K$ immediately, across $M_+$ next to the apex, whatever the pursuer does: the apex is usable exactly when the usable part is nonempty, and then $W(\mathbf 0)=0$. For $\nu\le\sin\alpha$ the apex is held by straight reversal (Proposition~\ref{prop:kind}(a)).

\section{Hamiltonian, terminal data and primary arcs}\label{sec:ham}
We set up Isaacs' method of characteristics \citep{isaacs1965} and integrate the characteristics that end on the usable part, the primaries, in closed form.

\subsection{Hamiltonian and saddle controls}
With costate $\boldsymbol\lambda=(\lambda_x,\lambda_y)^\top$, the Hamiltonian of~\eqref{eq:reduced} for the minimum-time payoff is
\begin{equation}\label{eq:Ham}
\begin{aligned}
H&=1+\boldsymbol\lambda^\top\mathbf u-v_p\Phi_v+\omega_p\Phi_\omega,\\
\Phi_v&=\lambda_y,\qquad \Phi_\omega=\lambda_xy-\lambda_yx .
\end{aligned}
\end{equation}
The evader minimizes and the pursuer maximizes $H$, so
\begin{equation}\label{eq:Hstar}
H^*(\x,\boldsymbol\lambda)=1-\Ve\Lambda+\Vp|\lambda_y|+\Om|\Phi_\omega|,
\end{equation}
where $\Lambda=\|\boldsymbol\lambda\|$, with saddle controls $\mathbf u^*=-\Ve\boldsymbol\lambda/\Lambda$, $v_p^*=-\Vp\sgn\Phi_v$ and $\omega_p^*=\Om\sgn\Phi_\omega$. Where a switching function vanishes, the maximizer is not unique, and we select an admissible one: if $\Phi_v=0$ and $\Phi_\omega\ne0$, either sign of $v_p$ with $|v_p|=\Vp$, so that the saturated turn rate lies in the bow tie; if $\Phi_\omega=0$, any $\omega_p$ with $|\omega_p|\le\kmax|v_p|$. The costate equations $\dot\lambda_x=\omega_p\lambda_y$, $\dot\lambda_y=-\omega_p\lambda_x$ rotate $\boldsymbol\lambda$ rigidly, so $\Lambda$ is constant on smooth arcs. We write $\boldsymbol\lambda=-\Lambda(\sin\beta,\cos\beta)^\top$, where $\beta$ is the evader's heading in the reduced frame, and $\psi=y\sin\beta-x\cos\beta$, so that $\Phi_\omega=-\Lambda\psi$. In retro-time $\tau=t_f-t$, with primes denoting $d/d\tau$, the characteristic system is
\begin{equation}\label{eq:retro}
\begin{aligned}
x'&=-\Ve\sin\beta-\omega_p^*y,\\ y'&=-\Ve\cos\beta+v_p^*+\omega_p^*x,\\
\beta'&=-\omega_p^*,\qquad \psi'=v_p^*\sin\beta,
\end{aligned}
\end{equation}
with $v_p^*=\Vp\sgn(\cos\beta)$ and $\omega_p^*=-\Om\sgn(\psi)$. The last equation follows from the first three, because the terms in $\Ve$ and in $\omega_p^*$ cancel. With constant controls,~\eqref{eq:retro} integrates in closed form. The saturated controls $|v_p|=\Vp$, $|\omega_p|=\Om$ have $|\kappa_p|=\kmax$, so the maximum in~\eqref{eq:Hstar} is attained by admissible Reeds--Shepp controls, as is the singular steering $\omega_p=0$ of Proposition~\ref{prop:singular}(ii). A characteristic is \emph{normal} if it carries the running cost, i.e., satisfies $H^*=0$ with~\eqref{eq:Hstar}, and \emph{abnormal} if the running cost is absent, i.e., it satisfies $H_A^*=H^*-1=0$ with $\boldsymbol\lambda\neq0$; abnormal characteristics trace semipermeable surfaces such as barriers \citep{isaacs1965}.

\begin{proposition}[Singular arcs]\label{prop:singular}
Let $\x$ lie in the interior of $K$, so that $y>0$. (i) No normal characteristic has a singular pursuer speed, $\Phi_v\equiv0$, on a time interval. (ii) A normal characteristic with singular steering, $\Phi_\omega\equiv0$, on a time interval lies on the front axis $x=0$, with $\boldsymbol\lambda=(0,\lambda_y)$, $\omega_p=0$, $|v_p|=\Vp$ and $(\Ve-\Vp)|\lambda_y|=1$. In particular, normal singular arcs do not exist for $\nu\le1$.
\end{proposition}
\begin{proof}
Along characteristics, $\dot\lambda_x=\omega_p\lambda_y$ and $\dot\lambda_y=-\omega_p\lambda_x$, and differentiating $\Phi_\omega$ along~\eqref{eq:reduced} with $\mathbf u^*=-\Ve\boldsymbol\lambda/\Lambda$ gives $\dot\Phi_\omega=-v_p\lambda_x$, because the evader terms $\lambda_xu_y-\lambda_yu_x$ cancel. (i) If $\lambda_y\equiv0$, then $\lambda_x\neq0$ and $\dot\lambda_y=0$ force $\omega_p=0$, whereas $\Phi_\omega=\lambda_xy\neq0$ forces $\omega_p=\pm\Om$. (ii) If $\Phi_\omega\equiv0$, then $\lambda_y\neq0$, since otherwise $\Phi_\omega=\lambda_xy=0$ would give $\boldsymbol\lambda=0$. Hence $|v_p|=\Vp$, and $\dot\Phi_\omega=0$ gives $\lambda_x=0$, then $\Phi_\omega=-\lambda_yx=0$ gives $x=0$, and $\dot\lambda_x=0$ gives $\omega_p=0$. Finally $H^*=1-(\Ve-\Vp)|\lambda_y|=0$, which has no solution when $\Ve\le\Vp$.
\end{proof}

\subsection{Terminal transversality}
At a regular terminal point $\x_f=\rho_f(\sin\alpha,\cos\alpha)^\top$ of $M_+$, the costate is normal to the ray and points into $K$, where $W$ increases:
\begin{equation}\label{eq:lamf}
\boldsymbol\lambda_f=\Lambda_f(-\cos\alpha,\sin\alpha)^\top,\qquad \beta_f=\alpha+\tfrac{\pi}{2}.
\end{equation}
Then $\psi_f=\rho_f$ and $\Phi_\omega(t_f)=-\Lambda_f\rho_f\neq0$. The steering switching function therefore does not vanish at termination, and the terminal turn rate is saturated: $\omega_p^*=-\Om$, and $v_p^*=-\Vp$, $\kappa_p^*=+\kmax$. The condition $H^*=0$ gives
\begin{equation}\label{eq:Lamf}
\Lambda_f=\frac{1}{\Ve-\Vp\sin\alpha-\Om\rho_f},
\end{equation}
which is positive and finite exactly on the $\UP$ and tends to $+\infty$ at the $\BUP$.

\subsection{Primary arcs in a rotating frame}
With constant controls $(-\Vp,-\Om)$, write $\theta=\alpha+\Om\tau$, where $\tau$ is retro-time. Integrating~\eqref{eq:retro} gives the primary family in closed form:
\begin{equation}\label{eq:primary}
\begin{aligned}
\psi(\tau)&=\rho_f-r(\sin\theta-\sin\alpha),\\
\eta(\tau)&=r(\cos\alpha-\cos\theta)-\Ve\tau,\\
x&=\psi\sin\theta+\eta\cos\theta,\quad y=\psi\cos\theta-\eta\sin\theta,
\end{aligned}
\end{equation}
with $\beta=\theta+\pi/2$ and $W=\tau$. The switching variable $\psi$ is the coordinate of $\x$ along $\mathbf e_r(\theta)$, and $\eta$ is its coordinate along $\mathbf e_\varphi(\theta)$, so the bearing is $\varphi=\theta+\operatorname{atan2}(\eta,\psi)$. Initially $\varphi'=\Om(1-\rho_B/\rho_f)<0$, and the characteristic enters $K$. The speed switch $\cos\beta=0$ occurs at $\theta=\pi$. The steering switch $\psi=0$ occurs when $\sin\theta=\sin\alpha+\rho_f/r$. A primary is \emph{retained} up to the first point where it stops carrying the value; for $\nu\le1$ this is its first contact with the front axis, where it meets its mirror image (Section~\ref{sec:slow}), and beyond that point it is discarded.

\begin{proposition}[Primaries]\label{prop:prim}
Let $\sin\alpha<\nu\le1$. (i) The map $(\rho_f,\tau)\mapsto\x$ has Jacobian determinant $\Om(\rho_B-\rho_f)>0$, so the family has no focal points. (ii) Every primary reaches the front axis before any switch of the pursuer's controls; up to that point the pursuer reverses at full speed and turns toward the evader at the maximum rate, $\lambda_y=\Lambda\sin\theta>0$ and $\Phi_\omega=-\Lambda\psi<0$.
\end{proposition}
\begin{proof}
We use normalized units. (i) By~\eqref{eq:primary}, $\partial\x/\partial\rho_f=(\sin\theta,\cos\theta)^\top$ and $\partial\x/\partial\tau=(x',y')^\top$ with $x'=(\psi-\nu)\cos\theta-\eta\sin\theta$ and $y'=-1+(\nu-\psi)\sin\theta-\eta\cos\theta$, so the Jacobian determinant is $y'\sin\theta-x'\cos\theta=\nu-\sin\theta-\psi=\nu-\sin\alpha-\rho_f=\rho_B-\rho_f$. (ii) Let $b=\rho_f+\sin\alpha<\nu\le1$, so that $\psi=b-\sin\theta$. Since $\eta(\alpha)=0$ and $\eta'=\sin\theta-\nu$, $\eta<0$ for $\theta>\alpha$ as long as $\sin\theta<\nu$. At the first zero of $\psi$, $\theta_\omega=\arcsin b<\pi/2$, we have $\sin\theta\le b<\nu$ before it, so $x=\eta\cos\theta_\omega<0$ there, whereas $x=\rho_f\sin\alpha>0$ at launch: the axis is reached before the steering switch. The speed switch $\cos\beta=0$ occurs only at $\theta=\pi$, later still. Up to the axis contact, $\psi>0$ and $\sin\theta>0$, which give the stated signs.
\end{proof}

All retained primaries are therefore the single arc~\eqref{eq:primary}, truncated at their first contact with the front axis; they do not return to $M_+$ before it (proof of Proposition~\ref{prop:Wform}).

\section{Game of kind}\label{sec:kind}
Whether the pursuer can prevent escape at all splits the parameter plane into three cases: it holds the whole cone, part of it, or none of it.

\begin{proposition}[Game of kind]\label{prop:kind}
Let $0<\alpha<\pi/2$, and in this proposition let $\Hset=\{\x\in K:V_P(\x)=+\infty\}$, the set of states from which a pursuer strategy can prevent the escape.
\begin{enumerate}[label=(\alph*)]
\item If $\nu\le\sin\alpha$, then $\Hset=K$. Straight reversal, $v_p=-\Vp$, $\omega_p=0$, keeps every trajectory that starts in $K$ inside $K$ for all time, whatever the evader does; hence $V_E=V_P=+\infty$ on $K$. This is one ordinary Reeds--Shepp control ($\kappa_p=0$) with no switching.
\item If $\sin\alpha<\nu\le1$, then $\emptyset\neq\Hset\neq K$. Theorems~\ref{thm:slow} and~\ref{thm:value}, proved independently of this classification, show that $\Hset$ is the region~\eqref{eq:esc} bounded by the barrier and the nonusable rays.

\item If $\nu>1$, then $\Hset=\emptyset$: for every $\varepsilon>0$ the evader has a strategy that guarantees
\begin{equation}\label{eq:fastbound}
t_f\le\frac{\rho}{\Ve-\Vp}+\varepsilon\qquad\text{for all }\x\in K
\end{equation}
against every pursuer control, and $V_E$ and $V_P$ are finite on $K$.
\end{enumerate}
\end{proposition}
\begin{proof}
(a) With $\omega_p=0$ and $v_p=-\Vp$, \eqref{eq:reduced} gives $\dot\x=\mathbf u+(0,\Vp)^\top$. The ball of radius $\Ve$ around $(0,\Vp)^\top$ lies in $K$, because its distance to either ray is $\Vp\sin\alpha\ge\Ve$. Hence $\dot\x(t)\in K$ for all $t$, and $\x(t)=\x(0)+\int_0^t\dot\x\in K$ because $K$ is a closed convex cone.

(b) By Corollary~\ref{cor:sandwich}, the relay gives $V_P=+\infty$ on the nonempty region~\eqref{eq:esc}, so $\Hset\ne\emptyset$. Since $\nu>\sin\alpha$, the usable part is nonempty, and from a point of it the evader exits immediately against every pursuer control (Section~\ref{sec:target}), so $V_P=0$ there and $\Hset\ne K$.

(c) At the apex the evader exits immediately (Section~\ref{sec:target}). Otherwise fix $\delta\in(0,\rho)$ with $\Om\delta<\Ve-\Vp$. While $\rho>\delta$, let the evader play $\mathbf u=-\Ve\mathbf e_r$; then $\dot\rho=-\Ve-v_p\cos\varphi\le-(\Ve-\Vp)$, so $\rho$ reaches $\delta$ within $\rho/(\Ve-\Vp)$ unless the state leaves $K$ earlier. From then on, let it play $\mathbf u=(0,-\Ve)$. While the state stays in $K$, $\dot\rho=-(\Ve+v_p)\cos\varphi\le-(\Ve-\Vp)\cos\alpha<0$, so $\rho\le\delta$ and $|x|\le\delta$, and therefore $\dot y=-\Ve-v_p-\omega_px\le-(\Ve-\Vp)+\Om\delta<0$. Since $y\le\delta$ at the start of this phase, the state leaves $K$, across a ray or the line $y=0$, within $\delta/(\Ve-\Vp-\Om\delta)$. The total time is at most $\rho/(\Ve-\Vp)+\delta/(\Ve-\Vp-\Om\delta)$, and choosing $\delta$ small enough gives~\eqref{eq:fastbound}; hence $V_E<+\infty$. The second phase uses a constant control, and in the first phase holding $-\Ve\mathbf e_r(\x(kh))$ on $[kh,(k+1)h)$ changes the bearing by $O(h/\delta)$ while $\rho\ge\delta$, so that $\dot\rho\le-(\Ve-\Vp)+O(h/\delta)$. This sampled evader, paired with an arbitrary pursuer strategy as at the beginning of Appendix~\ref{app:value}, gives the same bound up to $\varepsilon$ for small $h$; hence also $V_P<+\infty$.
\end{proof}

Case~(a) uses reversal, unavailable to a forward-only pursuer. In case~(b), holding does not require the separation to increase: for $\nu=1$, an evader running straight at a reversing pursuer can remain at a fixed reduced state on the held axis. In case~(c), the evader can close the range at speed at least $\Ve-\Vp$. Figure~\ref{fig:kind} shows the three cases in the $(\nu,\alpha)$ plane.

\begin{figure}[!htbp]
\centering
\includegraphics[width=0.8\textwidth]{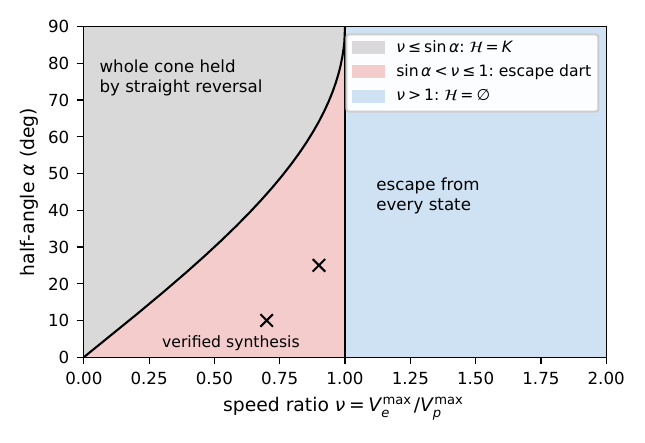}
\caption{Game of kind. For $\nu\le\sin\alpha$ straight reversal holds the whole cone; for $\sin\alpha<\nu\le1$ escape is confined to a dart bounded by a barrier (Theorems~\ref{thm:slow} and~\ref{thm:value}); for $\nu>1$ escape is possible from every state (Proposition~\ref{prop:kind}(c)). Crosses: the examples of Figs.~\ref{fig:slow} and~\ref{fig:levelsets}.}
\label{fig:kind}
\end{figure}

\section{Synthesis for the slower or equally fast evader}\label{sec:slow}
Throughout this section $\sin\alpha<\nu\le1$, the equal-speed case included; we construct the escape region and the escape time and verify the synthesis in feedback form.

\subsection{Primary family and the axis dispersal line}
By Proposition~\ref{prop:prim}(ii), the primaries issued from $\rho_f\in(0,\rho_B)$ on $M_+$ reach the front axis, where they meet their mirror images at the same retro-time. They arrive with $\beta\in(\pi/2,\pi)$, so $\lambda_x<0$ on the right. With $W_R$ and $W_L$ the values carried by the primaries ending on $M_+$ and $M_-$, $W=\min(W_L,W_R)$ has a concave kink on the axis. At such a kink, the viscosity \emph{subsolution test} requires $H^*(\x,\mathbf p)\ge0$ for every $\mathbf p$ in the superdifferential of $W$, the segment joining the two one-sided gradients \citep{bardi1997}; with the stationary operator written as $F=-H^*$, this is the usual condition $F\le0$ at upper test functions. The kink corresponds to the evader's choice of the side through which to escape, and we call
\begin{equation}\label{eq:De}
\De=\{(0,y)^\top:\ 0<y<y_I\}
\end{equation}
the axis dispersal line of the evader, with $y_I$ given below.

\begin{proposition}[Viscosity test on the axis]\label{prop:De}
At a point of the front axis reached by a primary and its mirror image with gradients $(\pm\lambda_x,\lambda_y)$, $\lambda_y>0$, the subsolution test, $H^*\ge0$ on the segment joining the two gradients, holds for every $\lambda_y>0$ if $\nu\le1$; if $\nu>1$, it holds if and only if $\lambda_y\le 1/(\Ve-\Vp)$. In particular it holds on all of $\De$ when $\nu\le1$.
\end{proposition}
\begin{proof}
At $x=0$, $\Phi_\omega=\lambda_xy$. On the segment $\mathbf p=(\xi\lambda_x,\lambda_y)$, $|\xi|\le1$, $H^*(\xi)=1-\Ve\sqrt{\xi^2\lambda_x^2+\lambda_y^2}+\Vp\lambda_y+\Om|\xi\lambda_x|y$ is even and concave in $\xi$ on $[0,1]$, so, by symmetry in $\xi$, its minimum over the segment is attained at $\xi=1$, where it vanishes, or at $\xi=0$, where it equals $1+(\Vp-\Ve)\lambda_y$.
\end{proof}

\subsection{The barrier and the escape region}
The abnormal characteristic issued from $B_+$ with $\boldsymbol\lambda\propto(-\cos\alpha,\sin\alpha)^\top$ satisfies $H_A^*=-\Ve\Lambda+\Vp|\lambda_y|+\Om|\Phi_\omega|=0$; we call it the \emph{$\BUP$ characteristic}. By~\eqref{eq:retro}, it follows the same state path as the limit of~\eqref{eq:primary} as $\rho_f\uparrow\rho_B$, i.e., \eqref{eq:primary} with $\psi=r(\nu-\sin\theta)$. Let $d=y\sin\alpha-x\cos\alpha$ be its distance from $M_+$. Then $d(0)=d'(0)=0$ and
\begin{equation}\label{eq:dpp}
d''(0)=\Vp\Om\cos\alpha>0,
\end{equation}
so the barrier always departs into the cone.

\begin{proposition}[Closure of the barrier]\label{prop:closure}
Let $\sin\alpha<\nu\le1$. The barrier $\Bbar_+$, the $\BUP$ characteristic issued from $B_+$, reaches the front axis at the first root $\theta_I\in(\alpha,\arcsin\nu)$ of
\begin{equation}\label{eq:FI}
\Gamma(\theta)=\nu\sin\theta-1+\cos\alpha\cos\theta-\nu(\theta-\alpha)\cos\theta,
\end{equation}
at the point $(0,y_I)$ with
\begin{equation}\label{eq:yI}
y_I=r\,\frac{\nu-\sin\theta_I}{\cos\theta_I},\qquad \tau_I=\frac{\theta_I-\alpha}{\Om},
\end{equation}
before any switch of the pursuer's controls. Moreover $\tau_I<\tau_c=r\cos\alpha/\Ve$, so $q=r\cos\alpha-\Ve\tau$ remains positive on the retained barrier; $\tau_c$ is the retro-time of the only cusp of the constant-control formula~\eqref{eq:primary} with $\rho_f=\rho_B$. We write $\Bbar_-$ for its mirror image, issued from $B_-$, and $\Bbar=\Bbar_+\cup\Bbar_-$.
\end{proposition}
\begin{proof}
With $\psi=r(\nu-\sin\theta)$ and $\eta=r(\cos\alpha-\cos\theta-\nu(\theta-\alpha))$, using $\sin^2\theta+\cos^2\theta=1$, the abscissa along the barrier is $x=\psi\sin\theta+\eta\cos\theta=r\Gamma(\theta)$ identically, so the axis contacts are the roots of $\Gamma$. At a root, $\eta=-\psi\tan\theta$ and $y=\psi\cos\theta-\eta\sin\theta=\psi/\cos\theta$. Here $\Gamma(\alpha)=\sin\alpha(\nu-\sin\alpha)>0$. For $\nu<1$, at $\theta_\omega=\arcsin\nu$ we have $\psi=0$ and $\eta=r\int_\alpha^{\theta_\omega}(\sin s-\nu)\,ds<0$, so $r\Gamma(\theta_\omega)=\eta\cos\theta_\omega<0$, and $\Gamma$ has a root in $(\alpha,\arcsin\nu)$. For $\nu=1$ that argument gives only $\Gamma(\pi/2)=0$, because $\theta_\omega=\pi/2$ and $\cos\theta_\omega=0$. Instead, $\Gamma'(\theta)=\sin\theta\,[\nu(\theta-\alpha)-\cos\alpha]$, so $\Gamma$ decreases on $(\alpha,\theta_c)$ and increases on $(\theta_c,\pi/2]$, where $\theta_c=\alpha+\cos\alpha/\nu$. For $\nu=1$, $\theta_c<\pi/2$ because $\pi/2-\alpha-\cos\alpha>0$ on $(0,\pi/2)$, and $\Gamma(\pi/2)=\nu-1=0$; hence $\Gamma(\theta_c)<0$, and $\Gamma$ has exactly one root in $(\alpha,\pi/2)$, which lies in $(\alpha,\theta_c)$. Along the barrier, $(x',y')=-\Om q(\sin\theta,\cos\theta)$ with $q=r\cos\alpha-\Ve\tau$, so the constant-control formula is regular for $\tau<\tau_c=r\cos\alpha/\Ve$, i.e., $\theta<\theta_c$, and its only cusp is at $\tau_c$. (Beyond the steering switch at $\theta=\arcsin\nu$ the continued characteristic follows a different control arc; the retained barrier ends at $\tau_I$, before that switch.) Finally $\tau_I<\tau_c$: if $\theta_c\ge\arcsin\nu$ (possible only for $\nu<1$), this follows from $\theta_I<\arcsin\nu$; otherwise $\theta_c<\pi/2$, $\Gamma$ decreases on $(\alpha,\theta_c)$ and $\Gamma(\theta_c)=\nu\sin\theta_c-1<0$, so $\theta_I<\theta_c$.
\end{proof}

The barrier, its mirror image and the two $\NUP$ rays bound a closed region of $K$, which contains $\Bbar$, $B_\pm$ and the $\NUP$. For $\sin\alpha<\nu\le1$ we take this region as the definition of the holding region $\Hset$ and set $W=+\infty$ on it; its game-theoretic meaning is established afterwards, by Theorems~\ref{thm:slow} and~\ref{thm:value}. The escape region is the dart
\begin{equation}\label{eq:esc}
\Eset=K\setminus\Hset=\text{region enclosed by }\UP_\pm,\ \Bbar_\pm,
\end{equation}
which contains the $\UP$, the apex and the axis segment $\De$, but not the barrier. $W$ jumps from finite values to $+\infty$ across $\Bbar$. Figure~\ref{fig:slow} shows two examples. In normalized units, $(\rho_B,y_I,\tau_I)=(0.5264,0.2516,0.3252)$ for $(\nu,\alpha)=(0.7,10^\circ)$ and $(0.4774,0.1695,0.4874)$ for $(0.9,25^\circ)$.

\begin{figure}[!htbp]
\centering
\includegraphics[width=\textwidth]{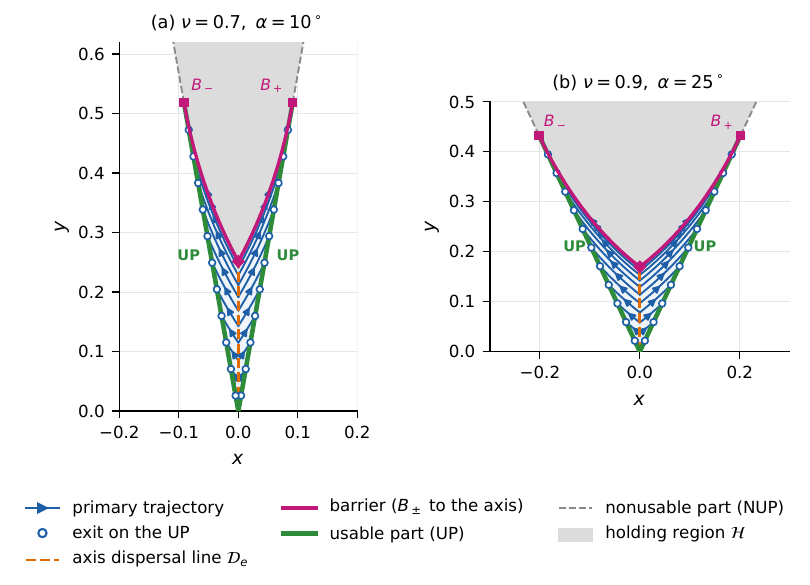}
\caption{Characteristic synthesis for the slower evader, verified in feedback form by Theorem~\ref{thm:slow}. (a) $\nu=0.7$, $\alpha=10^\circ$; (b) $\nu=0.9$, $\alpha=25^\circ$. Blue: primary trajectories, with arrowheads in forward time and open circles at their exits on the UP (green); magenta: barrier from the BUP points $B_\pm$ (squares) closing on the axis at $(0,y_I)$ (diamond); orange dashed: the evader's axis dispersal line $\De$; gray dashed: NUP; gray: holding region $\Hset$ ($W=\infty$).}
\label{fig:slow}
\end{figure}

\begin{proposition}[Escape time from a scalar equation]\label{prop:Wform}
Let $\sin\alpha<\nu\le1$ and let $\x=(x,y)$, $x\ge0$, lie in the right half of the escape region, but not on the $\UP$ or at the apex, where $W=0$. Write $\theta(\tau)=\alpha+\Om\tau$ and
\begin{equation}\label{eq:Gscalar}
G(\tau;\x)=(x+r)\cos\theta-y\sin\theta-r\cos\alpha+\Ve\tau .
\end{equation}
Then $W(\x)$ is the first positive root of $G(\cdot;\x)$, which is also its only root with $G_\tau>0$ on $(0,\tau_I)$. The launch parameter of the primary through $\x$ is
\begin{equation}\label{eq:rhof}
\rho_f=(x+r)\sin\theta+y\cos\theta-r\sin\alpha,
\end{equation}
with $\theta=\alpha+\Om W(\x)$. Moreover $G_\tau=\Om(\rho_B-\rho_f)>0$, so the root defines a smooth right-branch escape time $W_R$ on a neighbourhood of the right half of the escape region, with $\nabla W_R=\Lambda_f(-\cos\theta,\sin\theta)$, the costate~\eqref{eq:lamf}--\eqref{eq:Lamf} carried along the primary; $W=W_R$ for $x\ge0$, so $\nabla W=\nabla W_R$ for $x>0$. On the axis $W$ is not differentiable: the reflected left branch has gradient $\Lambda_f(\cos\theta,\sin\theta)$. Finally, $0<W(\x)<\tau_I$.
\end{proposition}
\begin{proof}
We use normalized units, in which $G_\tau=\rho_B-\rho_f$ below reads $\Om(\rho_B-\rho_f)$ in dimensional form. The proof has four steps.

\emph{(a) Scalar representation and local regularity.} By~\eqref{eq:primary}, a point of the primary with label $\rho_f$ at retro-time $\tau$ satisfies $x\cos\theta-y\sin\theta=\eta$ and $x\sin\theta+y\cos\theta=\psi$. With $\eta=\cos\alpha-\cos\theta-\nu\tau$, the first relation is $G(\tau;\x)=0$; with $\psi=\rho_f-\sin\theta+\sin\alpha$, the second is~\eqref{eq:rhof}. Differentiating~\eqref{eq:Gscalar}, $G_\tau=\nu-(x+1)\sin\theta-y\cos\theta=\rho_B-\rho_f$ at a root. Where $G_\tau>0$, the implicit function theorem makes the root a smooth function $W_R$ of $\x$, also across the axis, with $\nabla W_R=-\nabla_{\x}G/G_\tau=(-\cos\theta,\sin\theta)/(\rho_B-\rho_f)$, which is the costate~\eqref{eq:lamf}--\eqref{eq:Lamf} carried along the primary.

\emph{(b) Uniqueness of the regular root.} Write $(x+1)\cos\theta-y\sin\theta=R_P\cos(\theta+\gamma)$ with $R_P=\sqrt{(x+1)^2+y^2}$ and $\gamma=\operatorname{atan2}(y,x+1)\in(0,\pi/2)$. We consider $0<\tau<\tau_I$; since $\tau_I<\tau_c=\cos\alpha/\nu$ (Proposition~\ref{prop:closure}), $q(\tau)=\cos\alpha-\nu\tau>0$ there. A root satisfies $R_P\cos(\theta+\gamma)=q(\tau)>0$, so every root in this interval lies in the initial interval of $\tau$ on which $\theta+\gamma<\pi/2$, and there $G_{\tau\tau}=-R_P\cos(\theta+\gamma)<0$. Since $G(0;\x)=x\cos\alpha-y\sin\alpha<0$ inside the cone, $G$ is concave and starts negative on that interval, so it has at most one root with positive slope, and that root is its first positive root. Two primaries of the same family therefore cannot pass through the same point. A point of $\Bbar_+$ carries no regular root: there the root carried by the barrier has $\rho_f=\rho_B$, i.e., $G_\tau=0$, so by concavity it is the maximum of $G$.

\emph{(c) Axis contact before $\tau_I$.} Let $g=\rho_B-\rho_f>0$ and $q(\tau)=\cos\alpha-\nu\tau$. By~\eqref{eq:primary}, the primary with label $\rho_f$ is
\[
x=(\nu-g)\sin\theta-1+q\cos\theta=x_B-g\sin\theta,\qquad y=(\nu-g)\cos\theta-q\sin\theta,
\]
where $x_B=\Gamma(\theta)$ is the abscissa of the $\BUP$ characteristic. Hence $x_\tau=-g\cos\theta-q\sin\theta<0$ on $[0,\tau_I]$, because $q>0$ there ($\tau_I<\tau_c$, Proposition~\ref{prop:closure}) and $\theta<\pi/2$, while $x(0)=\rho_f\sin\alpha>0$ and $x(\tau_I)=-g\sin\theta_I<0$. Each primary therefore meets the axis exactly once on $(0,\tau_I)$, at a time $\tau_D(\rho_f)$. By the implicit function theorem $\tau_D$ is smooth, with $d\tau_D/d\rho_f=\sin\theta/(g\cos\theta+q\sin\theta)>0$, and $\tau_D\to0$ as $\rho_f\to0$, $\tau_D\to\tau_I$ as $\rho_f\to\rho_B$. A retained primary, stopped at its axis contact, carries values $W<\tau_I$. The contact height $y_D=(\cos\theta-\cos\alpha+\nu\tau)/\sin\theta$ increases along the contact curve, with $y_D'(\tau)=g/\sin\theta>0$, so the contacts sweep $\De$ from the apex to $c$.

\emph{(d) Coverage.} Let $D_+$ be the right half of the escape region, the open region bounded by $\UP_+$, $\De$ and the arc of $\Bbar_+$ from $B_+$ to $(0,y_I)$. Let $\Sigma=\{(\rho_f,\tau):0<\rho_f<\rho_B,\ 0<\tau<\tau_D(\rho_f)\}$, where $\tau_D$ is the axis-contact time, and let $\Pi(\rho_f,\tau)=\x$. By Proposition~\ref{prop:prim}(i), $\Pi$ is a local diffeomorphism; by (b) it is injective; so its image $\mathcal O$ is open and connected. $\mathcal O$ lies in the open right half of $K$: the distance from $M_+$ along a primary is $d=\cos\alpha-\cos(\alpha-\tau)-\rho_f\sin\tau+\nu\tau\cos\tau=d_B+(\rho_B-\rho_f)\sin\tau$, where $d_B$, the distance along $\Bbar_+$, has $d_B'=q\sin\tau>0$ for $0<\tau\le\tau_I$, since $q>0$ there and $\tau_I<\pi/2$, so $d>0$ up to the axis contact. As the parameters approach an edge of $\Sigma$, the image approaches $\UP_+$ ($\tau\to0$), $\De$ ($\tau\to\tau_D$), $\Bbar_+$ ($\rho_f\to\rho_B$, with $\tau<\tau_I$) or the apex ($\rho_f\to0$, where $\tau_D\to0$), all of which lie on $\partial D_+$. Hence, if $\Pi(p_n)\to\x\in D_+$, the bounded parameters $p_n$ have a limit point in $\Sigma$ itself, and $\x\in\mathcal O$ by continuity: $\mathcal O\cap D_+$ is closed in $D_+$. No point of $\Bbar_+$ lies in $\mathcal O$, by (b). Since $\mathcal O$ is connected, meets $D_+$ near $\UP_+$ and does not meet $\partial D_+$, $\mathcal O\subset D_+$; since $D_+$ is connected and $\mathcal O$ is open and closed in it, $\mathcal O=D_+$.
\end{proof}

\subsection{The corner of the holding region}
At $(0,y_I)$ the two barrier branches meet with costates, normalized to unit length,
\begin{equation}\label{eq:corner}
\boldsymbol\lambda_R=(-\cos\theta_I,\sin\theta_I)^\top,\qquad \boldsymbol\lambda_L=(\cos\theta_I,\sin\theta_I)^\top,
\end{equation}
pointing into $\Hset$. The corner is convex for $\Hset$, and two tests separate the information patterns.

\begin{proposition}[Corner tests]\label{prop:corner}
(i) Per-normal (discriminating) test: for every $\mathbf p$ in the cone spanned by $\boldsymbol\lambda_R$ and $\boldsymbol\lambda_L$, $\max_{(v_p,\omega_p)}\min_{\mathbf u}\mathbf p^\top\dot\x=-\Ve\|\mathbf p\|+\Vp|p_y|+\Om|p_x|y_I\ge0$. (ii) Common-control test: $\max_{(v_p,\omega_p)}\min_{\mathbf u}\min\{\boldsymbol\lambda_R^\top\dot\x,\boldsymbol\lambda_L^\top\dot\x\}=\Vp\sin\theta_I-\Ve<0$.
\end{proposition}
\begin{proof}
(i) With $\mathbf p=(\sin\vartheta,\cos\vartheta)$, the expression is $\chi(\vartheta)=\Vp\cos\vartheta+\Om y_I|\sin\vartheta|-\Ve$. On each half of the cone of normals, between the bisector $\vartheta=0$ and one of the normals~\eqref{eq:corner}, $\chi$ is a sinusoid $R_0\cos(\vartheta\mp\vartheta_0)-\Ve$, with $R_0=\sqrt{(\Vp)^2+(\Om)^2y_I^2}$, $\vartheta_0=\operatorname{atan2}(\Om y_I,\Vp)\in(0,\pi/2)$ and $|\vartheta\mp\vartheta_0|<\pi/2$, hence concave. It equals $\Vp-\Ve\ge0$ at the bisector and vanishes at the normal, because each branch is semipermeable, so $\chi\ge0$ on each half. (ii) For each $i$, the evader can choose $\mathbf u=-\Ve\boldsymbol\lambda_i$, so the left side equals $-\Ve+\max_{(v_p,\omega_p)}\min_i c_i$ with $c_{R,L}=\mp\cos\theta_I\,\omega_py_I-v_p\sin\theta_I$, and $\min_ic_i=-v_p\sin\theta_I-\cos\theta_I\,y_I|\omega_p|$. The maximum is $\Vp\sin\theta_I$, attained by straight reversal, $\omega_p=0$, $v_p=-\Vp$. Finally $\sin\theta_I<\nu$ because $\theta_I<\arcsin\nu$.
\end{proof}

The two tests correspond to different orders of choice. In (i) the inward normal is specified first, and the pursuer then chooses a control that protects that face. In (ii) the pursuer must hold one control, fixed before the evader's control is known, that protects both faces at once; no such control exists. A pursuer that reacts to the evader's current control does protect both faces. In normalized units, write $S_I=\sin\theta_I$ and $C_I=\cos\theta_I$, so that $S_I+C_Iy_I=\nu$ by~\eqref{eq:yI}. At $c$, with $v_p=-1$, $\dot\x=(u_x+\omega_py_I,\ u_y+1)^\top$, and the inward normal velocities are $S_I(1+u_y)\mp C_I(u_x+\omega_py_I)$. Choose $\omega_p=\operatorname{clip}(-u_x/y_I,-1,1)$; then $|u_x+\omega_py_I|=(|u_x|-y_I)^+$, and the smaller inward normal velocity is
\[
S_I(1+u_y)-C_I(|u_x|-y_I)^+.
\]
For $|u_x|\le y_I$ it is at least $S_I(1-\nu)\ge0$; otherwise it is at least $S_I+C_Iy_I-\nu=0$, since $C_I|u_x|-S_Iu_y\le\nu$. This is the equivalent control of the Filippov relay. The failure of (ii) shows that a held control cannot certify holding at $c$; Theorem~\ref{thm:value} shows that an evader whose strategy observes the pursuer's control escapes from $c$. At $\nu=1$ the per-normal test has equality also at the bisector; for $\nu>1$ it fails there, consistent with $\Hset=\emptyset$ (Proposition~\ref{prop:kind}(c)).

\subsection{Verification}
We verify the synthesis with the relay $v_p=-\Vp$, $\omega_p=-\Om\sgn x$ of Lemma~\ref{lem:relay} and the evader's characteristic feedback, in two stages. $W$ is the remaining escape time, so along an optimal play $t+W(\x(t))$ is constant. The pursuer's relay makes $t+W(\x(t))$ nondecreasing against every evader control, i.e., $W$ decreases at most at unit rate, and the evader's feedback makes it nonincreasing against every pursuer control, as long as the state stays in the escape region, where $W$ is finite and locally Lipschitz. Entry into $\Hset$, where $W=+\infty$, is handled by invariance rather than by differentiating $W$. These feedback guarantees give $V_E\le W\le V_P$ (Corollary~\ref{cor:sandwich}). A feedback construction does not by itself provide the opposite inequalities \citep{vinter2004}; away from the corner $c$ of $\Hset$ they follow from sample-and-hold counter-controls, which are well defined against every nonanticipative strategy of the opponent (Theorem~\ref{thm:value} and Appendix~\ref{app:value}). The proofs use normalized units and forward time $t$, whereas the characteristics are parametrized by retro-time $\tau$. Theorem~\ref{thm:slow} uses two appendix results, the local Lipschitz continuity of Lemma~\ref{lem:cont}(a) and the invariance of Lemma~\ref{lem:inv}. Both rest only on Propositions~\ref{prop:closure} and~\ref{prop:Wform}, Lemma~\ref{lem:relay}, the barrier geometry of Section~\ref{sec:slow} and the boundary inequalities stated in the proof below, not on Theorem~\ref{thm:slow} or~\ref{thm:value}, so these forward references are not circular. Propositions~\ref{prop:De} and~\ref{prop:corner}(i) are consistency checks of the synthesis with the viscosity and discriminating-viability conditions; the verification does not use them.

\begin{theorem}[Slower or equally fast evader]\label{thm:slow}
Let $\sin\alpha<\nu\le1$, let $\Hset$ and $\Eset=K\setminus\Hset$ be as in~\eqref{eq:esc}, and let $W$ be the escape time of Proposition~\ref{prop:Wform} on $\Eset$, with $W=0$ on the $\UP$ and at the apex, and $W=+\infty$ on $\Hset$.
\begin{enumerate}[label=(\roman*)]
\item \emph{Pursuer's guarantee.} If the pursuer plays $v_p=-\Vp$, $\omega_p=-\Om\sgn x$, with Filippov sliding on $x=0$, then $t_f\ge W(\x_0)$ for every initial state $\x_0\in K$ and every measurable evader control; in particular, the state never leaves $K$ from $\Hset$.
\item \emph{Evader's guarantee.} Let $\x_0\in\Eset$, and let the evader play, in the right half of $\Eset$, $\mathbf u=\Ve(\cos\theta,-\sin\theta)$ with $\theta=\alpha+\Om W(\x)$, and its mirror image in the left half; on the axis it selects one side once, e.g., the right one, and applies that branch on the closed half-plane $\{x\ge0\}$, without convexification. Then $t_f\le W(\x_0)$ for every measurable pursuer control.
\end{enumerate}
In particular, every play from $\x_0\in\Eset$ in which both feedbacks are used has escape time $W(\x_0)$. The identification of $W$ with the strategy values is deferred to Theorem~\ref{thm:value}.
\end{theorem}
\begin{proof}
We use the normalized units of~\eqref{eq:scale}, $r=\Vp=\Om=1$. In the right half of the escape region let $\theta=\alpha+W$, $S=\sin\theta$, $C=\cos\theta$, and
\[
g=\rho_B-\rho_f>0,\qquad q=\cos\alpha-\nu W>0,\qquad m=\rho_f+\sin\alpha-S>0,
\]
where $q>0$ because $W<\tau_I<\tau_c$ (Propositions~\ref{prop:closure} and~\ref{prop:Wform}), and $m=\psi>0$ is the absence of a steering switch (Proposition~\ref{prop:prim}(ii)). By~\eqref{eq:primary} and Proposition~\ref{prop:Wform},
\begin{equation}\label{eq:aux}
xS+yC=m,\quad (x+1)C-yS=q,\quad \nabla W_R=\frac{(-C,S)^\top}{g},\quad m+g=\nu-S ,
\end{equation}
where $W_R$ is the smooth right-branch escape time of Proposition~\ref{prop:Wform}, equal to $W$ for $x\ge0$.
Write $v_p=-1+\Delta v$ and $\omega_p=-1+\Delta\omega$. The bow tie $|\omega_p|\le|v_p|\le1$ gives $\Delta v\ge0$ and $\Delta\omega\ge0$. By~\eqref{eq:reduced}, the velocity equals the velocity under $(-1,-1)$ plus $\Delta\mathbf f=(\Delta\omega\,y,\ -\Delta v-\Delta\omega\,x)^\top$.

\emph{(i) Pursuer.} In the open right half of the escape region the relay plays $\Delta v=\Delta\omega=0$, the saddle control of the primaries. Since $H^*=0$ with $\boldsymbol\lambda=\nabla W$, $\dot W=\nabla W^\top\dot\x\ge-1$ for every evader control. During sliding on the axis, $\dot x=0$ and $\dot W=W_y\dot y$ with $W_y=S/g>0$ and $\dot y=u_y+1\ge1-\nu\ge0$ (here $\nu\le1$ is used), so $\dot W\ge0>-1$. Since $W=\min(W_L,W_R)$ is locally Lipschitz in the escape region (Lemma~\ref{lem:cont}(a)), $W(\x(t))+t$ is nondecreasing while the state stays there. If the state leaves $K$ from the escape region at $t_f$, then $W\to0$ at the exit point, and $t_f\ge W(\x_0)$. At the apex, where Lemma~\ref{lem:relay} does not apply, $W=0$ and there is nothing to prove. If it enters $\Hset$, it remains there, because $\Hset$ is positively invariant under the relay, and $t_f=+\infty$. The boundary inequalities behind this invariance are the following. On the barrier branches the inward normal $(-\cos\theta,\sin\theta)$ gives $\min_{\mathbf u}\boldsymbol\lambda^\top\dot\x=0$, because each branch is semipermeable; on the $\NUP$ rays the inward normal velocity is at least $\rho-\rho_B>0$ by~\eqref{eq:saddle_ray}; and on the axis inside $\Hset$, sliding gives $\dot y\ge1-\nu\ge0$, while leaving the axis puts the state on a side where the saturated control satisfies that side's inequality; at the corner $c$ these one-sided estimates are combined through the distance to the continued $\BUP$ characteristic. Lemma~\ref{lem:inv} converts these inequalities into invariance by means of signed distance functions.

\emph{(ii) Evader.} Let the pursuer use any admissible control. Under the evader's feedback $\mathbf u=\nu(C,-S)$ and the pursuer control $(-1,-1)$, $\dot W=-1$, because this is the saddle pair and $H^*=0$. The deviation contributes $\nabla W^\top\Delta\mathbf f=-(\Delta v\,S+\Delta\omega(xS+yC))/g$, and~\eqref{eq:aux} gives
\begin{equation}\label{eq:Wdot}
\dot W=-1-\frac{\Delta v\,S+\Delta\omega\,m}{g}\le-1.
\end{equation}
For $\rho_f$, differentiate~\eqref{eq:rhof} with $\theta=\alpha+W$: by~\eqref{eq:aux}, $d\rho_f=S\,dx+C\,dy+q\,dW$. Along the trajectory, $S\dot x+C\dot y=q-\Delta v\,C+\Delta\omega(C-q)$, using $yS-xC=C-q$. Substituting~\eqref{eq:Wdot} and $m+g=\nu-S$,
\begin{equation}\label{eq:rhodot}
\dot\rho_f=-\Delta v\Bigl(C+\frac{qS}{g}\Bigr)+\Delta\omega\,\frac{gC-q(\nu-S)}{g}\le0.
\end{equation}
The first coefficient is negative because $C,q,S,g>0$. For the second, $gC-q(\nu-S)=(C-q)(\nu-1/S)-xC/S$, where $x\ge0$, $\nu\le1<1/S$, and $C-q=\int_0^{W}(\nu-\sin(\alpha+s))\,ds\ge0$ because $\alpha+W<\theta_I<\arcsin\nu$ (here $\nu\le1$ is used twice).

These estimates are combined in the following order. (1) By~\eqref{eq:rhodot}, $\rho_f(\x(t))\le\rho_f(\x_0)<\rho_B$, so $g$ stays bounded below and the state cannot approach the barrier through the primary chart. (2) On the axis the selected right branch gives $\dot x=\nu C+\omega_py\ge\nu C-y=[g+S(1-\nu S)]/C>0$ against every pursuer control (here $\nu S<1$), and by continuity $\dot x>0$ in a strip $\{0\le x<\varepsilon_0\}$ along the relevant part of the axis; hence the state enters the right half immediately and does not return to the axis, the evader's control is continuous along the trajectory, and no convexification arises. (3) Only then is~\eqref{eq:Wdot} integrated: $W$ reaches zero within time $W(\x_0)$, and in the closure of the chart $W=0$ only on the $\UP$ and at the apex. (4) There the outward normal velocity is positive, by~\eqref{eq:saddle_ray} on the $\UP$ and by Section~\ref{sec:target} at the apex. The chart and the branch control extend smoothly across these edges, and the evader keeps the extended control for a short time after the contact; by continuity the outward margin persists, uniformly in the pursuer control, so points outside $K$ follow the contact immediately and the contact time is the strict-exit time $t_f$. Reaching a nonusable boundary point would not suffice.
\end{proof}

\begin{corollary}[Strategy values]\label{cor:sandwich}
Let $\sin\alpha<\nu\le1$. The relay of Theorem~\ref{thm:slow} defines a nonanticipative pursuer strategy $B^*\in\mathcal S_P$, and the evader's feedback defines a nonanticipative evader strategy $A^*\in\mathcal S_E$. Hence $V_E\le W\le V_P$ on $K$; in particular $V_P=+\infty$ on $\Hset$.
\end{corollary}
\begin{proof}
By Lemma~\ref{lem:relay}, $B^*[\mathbf u]$, the pursuer control along the unique relay solution, is admissible and nonanticipative, and the pursuer part of the proof of Theorem~\ref{thm:slow} applies to that solution, so $J(\mathbf u,B^*[\mathbf u])\ge W(\x_0)$ for every measurable $\mathbf u$ and $V_P\ge W$; from the apex, where $W=0$, any fixed admissible control may serve as $B^*$. For the evader, each one-sided branch $W_R,W_L$ is smooth on its characteristic chart, including its continuation across the axis, because $G_\tau=\Om(\rho_B-\rho_f)>0$ (Proposition~\ref{prop:Wform}). The escape time $W=\min(W_R,W_L)$ is locally Lipschitz and has a concave kink on $\De$. The selected branch feedback is Lipschitz on compact subsets of its chart with $\rho_f\le\bar\rho<\rho_B$. Along the closed loop, $\rho_f$ does not increase and the state does not return to the axis (proof of Theorem~\ref{thm:slow}), so for every measurable pursuer control $\mathbf a$ the closed loop has a unique Carath\'eodory solution up to the exit. Let $A^*[\mathbf a](t)$ be the feedback along it up to the contact with the $\UP$ or the apex, then the extended branch control while the state remains in the extended chart, as in step~(4) of the proof of Theorem~\ref{thm:slow}, so that points outside $K$ follow the contact immediately, and an arbitrary fixed control afterwards. Inputs that agree a.e.\ on $[0,s]$ give the same solution and the same output on $[0,s]$, so $A^*$ is nonanticipative, and $J(A^*[\mathbf a],\mathbf a)\le W(\x_0)$ by Theorem~\ref{thm:slow}(ii). On $\Hset$ the bound $V_E\le W$ is automatic, so $V_E\le W$ on $K$.
\end{proof}

\begin{theorem}[Value of the game]\label{thm:value}
Let $\sin\alpha<\nu\le1$, and let $c=(0,y_I)^\top$ be the corner of $\Hset$. Then $V_P=W$ on $K$ and $V_E=W$ on $K\setminus\{c\}$, so the game has a value on $K\setminus\{c\}$, equal to the escape time $W$ of Theorem~\ref{thm:slow}. At the corner the two strategy values differ:
\[
V_E(c)=\tau_I<V_P(c)=+\infty .
\]
For $\nu\le\sin\alpha$, $V_E=V_P=+\infty$ on $K$.
\end{theorem}
\begin{proof}
For $\nu\le\sin\alpha$, the constant control of Proposition~\ref{prop:kind}(a) is a pursuer strategy and also defeats every evader strategy. For $\sin\alpha<\nu\le1$ the proof is given in Appendix~\ref{app:value}; we outline it. Against a pursuer strategy, the evader holds the feedback of Theorem~\ref{thm:slow} constant over short sampling periods, which costs $O(h)$ in the decrease of $W$ and gives $V_P\le W$. Against an evader strategy, the pursuer samples its relay. As the period tends to zero, the sampled trajectories converge to relay solutions of Lemma~\ref{lem:relay}, because the dynamics contain no product of the two players' controls, and along relay solutions the lower semicontinuous envelope of $W$ decreases at most at unit rate; this gives $V_E\ge W$ off $c$ and $V_E(c)\ge\tau_I$. At $c$, the evader chooses the barrier branch to cross from the germ at $t=0$ of the pursuer's control and follows that branch's characteristic control; the scalar equation~\eqref{eq:Gscalar} serves as an exact certificate that the state leaves $K$ before $\tau_I$ (Lemma~\ref{lem:corner}).
\end{proof}

The corner strategy of Lemma~\ref{lem:corner} crosses the left branch if the pursuer begins with the control $(-\Vp,-\Om)$ that protects the right branch, and the right branch otherwise. Its choice depends on the pursuer's control on every interval $(0,s)$, $s>0$, however short, which nonanticipative strategies without delay allow; a strategy with a positive delay cannot make it. The relay and this strategy admit no consistent pair of controls from $c$: against an attack on the right branch the relay plays $(-\Vp,-\Om)$ at once, which selects the attack on the left branch, and conversely. The per-normal test of Proposition~\ref{prop:corner}(i) gives discriminating viability \citep{cardaliaguet1996}, which against an opponent's strategy holds only approximately; at $c$ the approximation cannot be made exact. Off $c$, Lemma~\ref{lem:samp} supplies any finite guarantee below $W$, with a sampling period depending on the initial state. For $\nu\le1$, Proposition~\ref{prop:singular} and the unique primary representation exclude further singular surfaces in the verified synthesis.

\begin{remark}[The equal-speed case]\label{rem:nu1}
For $\nu=1$, the sliding bound $\dot y\ge\Vp-\Ve$ and the per-normal corner test at the bisector become equalities. With $\mathbf u=(0,-\Ve)$ the relay keeps a held axis state fixed; the verification requires only nonstrict inequalities. For $\alpha=30^\circ$, in normalized units, $y_I=0.18533$ and $\tau_I=0.68069$.

\end{remark}

\begin{remark}[Information pattern]\label{rem:pattern}
Lemmas~\ref{lem:samp} and~\ref{lem:evs} use only sampled state feedback: on $[kh,(k+1)h)$ each player's control is fixed by $\x(kh)$, without knowledge of the opponent's current control. This is the information pattern of \citet{lewin1979conic}, made precise in the spirit of the positional strategies of \citet{krasovskii1988}. Each guarantee holds against every measurable control of the opponent, and therefore whatever the opponent observes. Hence, at every $\x_0\in K\setminus\{c\}$, $W(\x_0)$ is the escape time that each player can guarantee to arbitrary precision without observing the opponent's control: the sampled relay keeps the state in $K$ on $[0,T]$ for every $T<W(\x_0)$, and from $\Eset$ the sampled evader feedback forces exit by $W(\x_0)+\varepsilon$, with sampling periods depending on $\x_0$, $T$ and $\varepsilon$. Only the corner depends on the information pattern. There the sampled relay still guarantees every horizon $T<\tau_I$, whereas the evader strategy of Lemma~\ref{lem:corner} reads the germ of the pursuer's control and is not a sampled state feedback; whether a sampled evader can leave $K$ from $c$ by $\tau_I+\varepsilon$ is open.
\end{remark}

\section{Numerical comparison}\label{sec:num}
We compare the closed-form escape time $W$ of Theorem~\ref{thm:slow} with an independent numerical approximation of the reach-time equation $\partial_t\zeta=\min\{0,\tilde H(\x,\nabla\zeta)\}$, $\tilde H(\x,\mathbf p)=-\Ve\|\mathbf p\|+\Vp|p_y|+\Om|p_xy-p_yx|$, initialized by the signed distance to the cone boundary, positive inside $K$ \citep{mitchell2005}. It is computed with fifth-order WENO differences \citep{jiang2000}, a local Lax--Friedrichs flux and third-order TVD Runge--Kutta time stepping \citep{shu1988} at CFL number $0.45$. The computed escape time $T$ is the first $t$ with $\zeta<0$, which mirrors the strict-exit convention; it is obtained by linear interpolation of $\zeta$ within the time step in which $\zeta$ first becomes negative. Outside the computational box, $\zeta$ is extrapolated linearly to the ghost nodes. A node not reached within the horizon is evidence for, not proof of, $W=+\infty$. The solver uses no characteristic information. At the comparison points, $W$ is computed by Newton's method on the primary family, initialized from a dense sampling of the family, and accepted only on the retained part of an arc; by Proposition~\ref{prop:Wform} it is the first root of~\eqref{eq:Gscalar}.

\begin{figure}[!htbp]
\centering
\includegraphics[width=\textwidth]{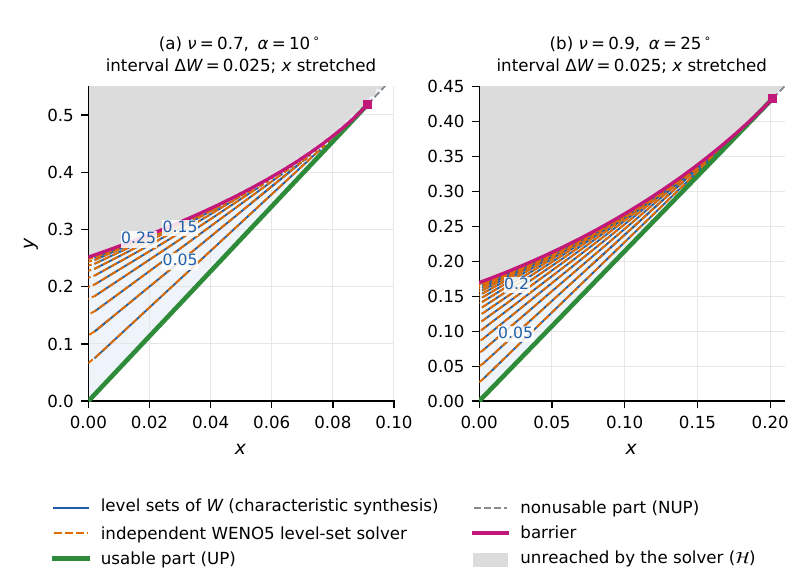}
\caption{Level sets of the escape time $W$ of Theorem~\ref{thm:slow} (solid blue) and of the independent WENO5 solution at $h=0.002$ (dashed orange) in the right half of the cone, at the contour interval $\Delta W=0.025$; labels give the time on selected curves. Gray: nodes not reached by the solver within the horizon, numerical evidence of the holding region. The horizontal axes are stretched. (a) $\nu=0.7$, $\alpha=10^\circ$; (b) $\nu=0.9$, $\alpha=25^\circ$.}
\label{fig:levelsets}
\end{figure}

For $(\nu,\alpha)=(0.9,25^\circ)$, $(0.7,10^\circ)$ and the equal-speed case $(1,30^\circ)$, the solver is run at $h=0.004$, $0.002$ and $0.001$ on boxes aligned with the coarsest grid, so that the three solutions are compared at the same physical points without interpolation. These are the $317$, $120$ and $510$ coarse nodes in the escape region with $x\ge0.012$ whose distance from the rays and the barrier exceeds $0.012$; the escape/holding classification is compared at all coarse nodes outside the same bands, where a node is misclassified when the solver reaches it and $W=+\infty$, or leaves it unreached and $W<+\infty$. The boxes are $[-0.4,0.4]$, $[-0.2,0.2]$ and $[-0.44,0.44]$ in $x$, all with $y\in[-0.096,0.704]$; the horizon is $1$, and by symmetry the comparison uses the right half.

\begin{table}[!h]
\centering\small
\caption{Level-set solver against the closed-form escape time $W$ at fixed comparison points: median relative error $|T-W|/\max(W,0.05)$, largest absolute error, signed mean error, and misclassified escape/holding nodes among those compared.}
\label{tab:validation}
\begin{tabular}{lccccc}
\toprule
$(\nu,\alpha)$ & $h$ & median rel. & max abs. & mean & misclassified\\
\midrule
$(0.9,25^\circ)$ & 0.004 & $7.9\times10^{-7}$ & $1.3\times10^{-6}$ & $1.1\times10^{-7}$ & 0/5996\\
 & 0.002 & $1.5\times10^{-7}$ & $4.2\times10^{-8}$ & $1.5\times10^{-8}$ & 0/5996\\
 & 0.001 & $4.4\times10^{-8}$ & $9.3\times10^{-9}$ & $4.2\times10^{-9}$ & 0/5996\\
$(0.7,10^\circ)$ & 0.004 & $2.1\times10^{-6}$ & $4.1\times10^{-6}$ & $1.5\times10^{-8}$ & 0/1675\\
 & 0.002 & $3.8\times10^{-7}$ & $5.5\times10^{-8}$ & $2.7\times10^{-8}$ & 0/1675\\
 & 0.001 & $1.0\times10^{-7}$ & $1.4\times10^{-8}$ & $7.0\times10^{-9}$ & 0/1675\\
$(1,30^\circ)$ & 0.004 & $5.7\times10^{-7}$ & $6.7\times10^{-6}$ & $4.2\times10^{-8}$ & 0/7650\\
 & 0.002 & $1.0\times10^{-7}$ & $3.3\times10^{-8}$ & $1.2\times10^{-8}$ & 0/7650\\
 & 0.001 & $2.8\times10^{-8}$ & $7.9\times10^{-9}$ & $3.3\times10^{-9}$ & 0/7650\\
\bottomrule
\end{tabular}
\end{table}

At each refinement both the median relative error and the largest error decrease (Table~\ref{tab:validation}): the largest error falls from $6.7\times10^{-6}$ at $h=0.004$ to $1.4\times10^{-8}$ at $h=0.001$, and no compared node is misclassified at any spacing. The decrease is not uniform, by factors of about $30$ to $200$ in the first halving and about $4$ in the second, so we do not assign an order of convergence. Errors this small are consistent with the smooth escape-time field on the primary chart; they do not imply comparable accuracy inside the excluded bands, where $W$ has a kink on the axis and jumps to $+\infty$ across the barrier. These comparisons concern $W$ away from those bands; they do not test the corner value $V_E(c)$, which is established analytically by Lemma~\ref{lem:corner} and Theorem~\ref{thm:value}. At the scale of Fig.~\ref{fig:levelsets} the two sets of level curves coincide, and the unreached set of the solver reproduces the holding region bounded by the barrier.

The closed-form arcs agree with a numerical integration of~\eqref{eq:retro} to $10^{-12}$, the Hamiltonian residuals and the subsolution test on $\De$ hold to rounding, and the identities used in the proofs of Theorem~\ref{thm:slow} and Lemma~\ref{lem:corner} were checked symbolically and on random states and controls; the scripts are available with the rest of the numerical material (see Data availability).

\section{Comparison with the conic game for a Dubins car}\label{sec:lewin}
The forward-only pursuer of \citet{lewin1979conic} has usable-ray length $r(\nu+\sin\alpha)$, compared with $r(\nu-\sin\alpha)$ here: reversal shortens it by $2r\sin\alpha$ and creates holding regions in convex cones. Its boundary radial and angular objectives conflict, producing a state-constraint arc; the reverse-turning control serves both objectives in the present synthesis. On the front axis, our pursuer backs away rather than choosing a turning side. This gives the holding region for $\nu\le1$.

Lewin and Olsder exclude knowledge of the opponent's current control, whereas our strategy values allow either player's strategy to react to it (Theorems~\ref{thm:slow} and~\ref{thm:value}). Off the corner $c$ this difference does not matter: by Remark~\ref{rem:pattern}, the escape time $W$ is also guaranteed by sampled state feedback, which uses only their information. Table~\ref{tab:lewin} lists the features of the two syntheses.

\begin{table}[!htbp]
\centering
\caption{Conic escape game for $0<\alpha<\pi/2$: Dubins pursuer (DC) of \citet{lewin1979conic} and Reeds--Shepp pursuer (RS, this paper); $\nu=1/w$ in their notation. Entries list features that occur in the respective syntheses for some parameter values, not at every pair $(\nu,\alpha)$.}
\label{tab:lewin}
\small
\renewcommand{\arraystretch}{1.15}\setlength{\tabcolsep}{5pt}
\begin{tabularx}{\linewidth}{@{}l>{\raggedright\arraybackslash\hsize=0.8\hsize}X>{\raggedright\arraybackslash\hsize=1.2\hsize}X@{}}
\toprule
 & DC & RS (this paper)\\
\midrule
Pursuer speed & $v_p=\Vp$ & $|v_p|\le\Vp$\\
Usable part & $\rho<r(\nu+\sin\alpha)$ & $\rho<r(\nu-\sin\alpha)$\\
Apex usable & all $\nu$ & $\nu>\sin\alpha$\\
Sliding arc on NUP & yes & no\\
Holding region & none reported & $K$ if $\nu\le\sin\alpha$; all but a dart if $\sin\alpha<\nu\le1$; none if $\nu>1$\\
Closed barrier & absent for $0<\alpha<\pi/2$ & $\sin\alpha<\nu\le1$\\
Evader's axis dispersal line & yes & yes ($\De$)\\
Pursuer's axis dispersal line & yes & no\\
Equivocal surface & yes & no\\
Transition surfaces & yes & no\\
Information pattern & no current-control knowledge & either player may react; values differ only at $c$\\
\bottomrule
\end{tabularx}
\end{table}

\section{Conclusion and future work}\label{sec:concl}
Reverse motion yields a global escape/holding classification and, for a slower or equally fast evader, a closed-form synthesis verified in feedback form. Sample-and-hold counter-controls identify its escape time with the value of the game in nonanticipative strategies at every state except the corner of the holding region; off the corner the same escape time is guaranteed, to arbitrary precision, by sampled state feedback that does not observe the opponent's control. At the corner the strategy values differ: the pursuer holds forever when its strategy reacts to the evader, whereas the evader's strategy value is $\tau_I$ when its strategy reacts to the pursuer. The constructed evader strategy forces a strict exit before $\tau_I$ against every pursuer control, with worst-case supremum $\tau_I$. An independent level-set solver agrees with the escape time on fixed comparison points.

Three questions remain. The evader's corner strategy uses the germ of the pursuer's control at $t=0$; for strategies with a positive delay it is not admissible, and the corner value for that information pattern is open. For $\nu>1$ the holding region is empty, but the escape time has a richer structure, with a universal line on the front axis and a barrier ending at a cusp, whose global verification requires further work. Extensions to half-planes, nonconvex cones and finite sensor range are separate problems.

\section*{Statements and Declarations}
\paragraph{Competing interests} The author declares no competing interests.
\paragraph{Author contributions} U.~Ruiz is the sole author of this work.
\paragraph{Funding} This research did not receive any specific grant from funding agencies in the public, commercial, or not-for-profit sectors.
\paragraph{Data availability} The solvers, raw numerical grids, run settings and scripts that reproduce every computation, table and figure are available from the author upon reasonable request.
\paragraph{Use of generative AI} During the preparation of this work the author used Claude (Anthropic) in order to edit the text, check derivations and numerical consistency, draft parts of the proofs, and write and run code for the numerical computations and figures. After using this tool, the author reviewed and edited the content as needed and takes full responsibility for the content of the published article.

\appendix
\section{Proof of Lemma~\ref{lem:relay}}\label{app:relay}
\begin{proof}[Proof of Lemma~\ref{lem:relay}]
\emph{Existence.} The right-hand side is measurable in $t$ and, for fixed $t$, upper semicontinuous in $\x$ with nonempty compact convex values bounded on compact sets, so solutions exist \citep{filippov1988}; since the velocity is bounded by $\nu+1+\|\x\|$, they do not blow up in finite time.

\emph{Uniqueness.} Let $\x_1$, $\x_2$ be two solutions with the same $\mathbf u$ and measurable selections $\varsigma_i(t)\in\operatorname{Sign}(x_i(t))$, while $y_1,y_2>0$. The common evader control and the term $(0,1)^\top$ cancel in the difference, $\dot\x_1-\dot\x_2=(-\varsigma_1y_1+\varsigma_2y_2,\ \varsigma_1x_1-\varsigma_2x_2)^\top$, and
\[
\tfrac12\tfrac{d}{dt}\|\x_1-\x_2\|^2=(\varsigma_2-\varsigma_1)(x_1y_2-x_2y_1).
\]
This vanishes when $x_1$ and $x_2$ have the same strict sign, since then $\varsigma_1=\varsigma_2$, or both vanish, since then $x_1y_2-x_2y_1=0$. If $x_1>0>x_2$, then $\varsigma_1=1$, $\varsigma_2=-1$, and it equals $2(x_2y_1-x_1y_2)<0$, because $y_1,y_2>0$. If $x_1>0=x_2$, then $\varsigma_1=1$, and it equals $(\varsigma_2-1)x_1y_2\le0$, because $y_2>0$. The remaining cases follow by exchanging the indices or by the reflection $x\mapsto-x$. Hence $\|\x_1-\x_2\|$ does not increase, and two solutions with the same initial state coincide.

\emph{Turn rate.} Where $x(t)\ne0$, $\varsigma=\sgn x$ and $\omega_p=-\sgn x$. An absolutely continuous function has zero derivative at almost every point of its zero set. Hence, for almost every $t$ with $x(t)=0$, $0=\dot x=u_x-\varsigma y$, so $\varsigma=u_x/y$, $\omega_p=-\varsigma=-u_x/y$, and $|\omega_p|\le1=|v_p|$. Residence on the axis (sliding) therefore requires $|u_x|\le y$; where the evader's lateral speed cannot be cancelled, the state crosses the axis, and on the null set of such crossing times no turn rate needs to be assigned. The turn rate so defined is measurable, $(v_p,\omega_p)$ lies in the bow tie, and $\x$ is the ordinary solution of~\eqref{eq:reduced} under $(\mathbf u,v_p,\omega_p)$.

\emph{Nonanticipativity.} If two evader controls coincide on $[0,t]$, uniqueness gives the same trajectory on $[0,t]$, and hence the same pursuer control there.
\end{proof}

\section{Proof of Theorem~\ref{thm:value}}\label{app:value}
The proof has four parts. Lemmas~\ref{lem:cont}--\ref{lem:mono} show that the relay keeps $\Hset$ invariant and that the finite trace $W_c$ defined below decreases at most at unit rate along relay solutions. Lemmas~\ref{lem:samp} and~\ref{lem:evs} pair sampled counter-controls with arbitrary nonanticipative strategies. The final argument matches the resulting bounds off the corner, and Lemma~\ref{lem:corner} treats the corner separately. Lemmas~\ref{lem:cont} and~\ref{lem:inv} use only Propositions~\ref{prop:closure} and~\ref{prop:Wform}, Lemma~\ref{lem:relay} and the geometry of the barrier; Lemma~\ref{lem:mono} combines Lemma~\ref{lem:inv} with the local estimate $\frac{d}{dt}W(\x(t))\ge-1$ from the pursuer part of the proof of Theorem~\ref{thm:slow}. Throughout, $\sin\alpha<\nu\le1$ and normalized units are used. As in~\eqref{eq:esc}, $\Eset=K\setminus\Hset$ is the escape region, which contains the $\UP$, the apex and $\De$; we write $\Hset_*=\Hset\setminus\Bbar$ for the holding states off the barrier, including the $\NUP$ rays; the ambient interior is $\operatorname{int}\Hset=\Hset_*\setminus\NUP$; and, for $P\in\Bbar_+$, we write $\tau_P\in[0,\tau_I]$ for its retro-time along the $\BUP$ characteristic, so that $\tau_{B_+}=0$ and $\tau_P=\tau_I$ when $P=c$, with mirror images on $\Bbar_-$. Let $W_c=W$ on $\Eset$, $W_c(P)=\tau_P$ on $\Bbar$ and $W_c=+\infty$ on $\Hset_*$, the lower semicontinuous envelope of $W$ relative to $K$, and let $W_\sharp=W$ on $K\setminus\{c\}$ and $W_\sharp(c)=\tau_I$. A \emph{relay solution} is the solution of Lemma~\ref{lem:relay} for some measurable evader control. The three functions differ only on the barrier:
\begin{center}
\small
\begin{tabular}{lccc}
\toprule
State & $W$ & $W_c$ & $W_\sharp$\\
\midrule
$\x\in\Eset$ & $W(\x)$ & $W(\x)$ & $W(\x)$\\
$P\in\Bbar\setminus\{c\}$ & $+\infty$ & $\tau_P$ & $+\infty$\\
$c$ & $+\infty$ & $\tau_I$ & $\tau_I$\\
$\x\in\Hset_*$ & $+\infty$ & $+\infty$ & $+\infty$\\
\bottomrule
\end{tabular}
\end{center}
$W_c$ assigns to each barrier point the limit of $W$ from the escape side (Lemma~\ref{lem:cont}(b)); it is lower semicontinuous and therefore suited to limits of trajectories. In particular $W_c(B_\pm)=0$, although $B_\pm$ are held states with $W(B_\pm)=+\infty$. $W_\sharp$ records the guarantee recovered by the sampled pursuer (Lemma~\ref{lem:samp}). Neither is a statement about the value of the game on $\Bbar$. The argument cannot be replaced by continuity of $W$ along relay solutions: $W$ jumps to $+\infty$ at the barrier, and the finite trace $W_c$ together with invariance is what controls trajectories that reach it. Figure~\ref{fig:corner} shows these sets for $\nu=0.9$, $\alpha=25^\circ$.

\begin{figure}[!htbp]
\centering
\includegraphics[width=\textwidth]{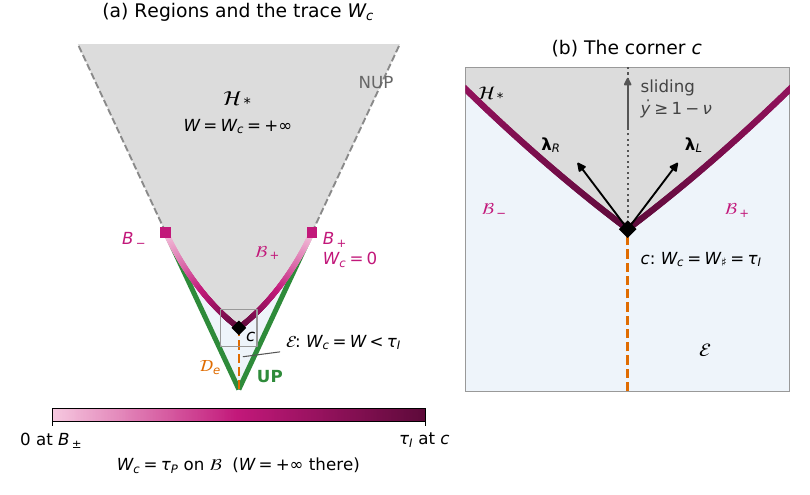}
\caption{Holding region for $\nu=0.9$, $\alpha=25^\circ$ (normalized units). (a) Escape region $\Eset$ (light), holding states off the barrier $\Hset_*$ (gray), $\UP$ (green), $\NUP$ (gray dashed), BUP points $B_\pm$ (squares), axis dispersal line $\De$ (orange dashed) and barrier branches $\Bbar_\pm$, colored by the escape-side trace $W_c=\tau_P$, which decreases from $\tau_I$ at the corner $c$ to $0$ at $B_\pm$; $W=+\infty$ on $\Bbar$. The box marks the enlargement in (b). (b) The corner $c$, with the inward normals $\boldsymbol\lambda_R$, $\boldsymbol\lambda_L$ of~\eqref{eq:corner} and the sliding direction of the relay on the axis inside $\Hset$.}
\label{fig:corner}
\end{figure}

\smallskip\noindent\emph{Sample-and-hold controls against strategies.} Let one player hold, on $[kh,(k+1)h)$, a constant control chosen from the state at $kh$. For every nonanticipative strategy of the other player there is exactly one pair of controls consistent with both rules: the held control on $[0,h)$ is fixed by $\x_0$; extending it arbitrarily beyond $h$ and applying the opponent's strategy gives the opponent's control on $[0,h)$, which by nonanticipativity does not depend on the extension, and hence the state at $h$; the construction proceeds by induction. No fixed point of two instantaneous strategies is needed. This is the usual pairing of strategies with delay with nonanticipative strategies \citep{bardi1997}.

\begin{lemma}[Boundary behaviour of $W$]\label{lem:cont}
(a) $W$ is continuous on $\Eset$ and locally Lipschitz relative to $K$: each point of $\Eset$ has a neighbourhood $U$ such that $U\cap K\subset\Eset$ and $W$ is Lipschitz on $U\cap K$. (b) If $P\in\Bbar_+$ and $\x\to P$ with $\x\in\Eset$, then $W(\x)\to\tau_P$; in particular $W\to0$ at $B_\pm$ and $W\to\tau_I$ at $c$.
\end{lemma}
\begin{proof}
(a) By Proposition~\ref{prop:Wform} and the implicit function theorem, the right-branch value $W_R$ is smooth on an open set containing $\Eset\cap\{x\ge0\}$; this neighbourhood includes a strip across the axis and neighbourhoods of the apex and of the $\UP$. On $\De$, $W=\min(W_L,W_R)$ with $W_L(x,y)=W_R(-x,y)$. Since $\Hset$ is closed, $\Eset$ is open relative to $K$; near the $\UP$ the smooth extension of $W_R$ is used, and only its restriction to $K$ is claimed. Near the apex, $W(x,y)=W_R(|x|,y)$, where $W_R$ is smooth because $G_\tau(0;\mathbf 0)=\rho_B>0$, and $x\mapsto|x|$ is Lipschitz; hence $W$ is locally Lipschitz relative to $K$ there as well.
(b) Let $G_P(\tau)=G(\tau;P)$. Since $G_{\tau\tau}=-(x+1)\cos\theta+y\sin\theta$, $G_P''+G_P=-q$ with $q(\tau)=\cos\alpha-\nu\tau$, and $G_P(\tau_P)=G_P'(\tau_P)=0$ because $P$ lies on the limiting primary $\rho_f=\rho_B$. Hence
\begin{align*}
G_P(\tau)&=-\int_\tau^{\tau_P}\sin(\tau'-\tau)\,q(\tau')\,d\tau'\quad(\tau<\tau_P),\\
G_P(\tau)&=-\int_{\tau_P}^{\tau}\sin(\tau-\tau')\,q(\tau')\,d\tau'\quad(\tau>\tau_P).
\end{align*}
On $[0,\tau_I+\delta]$, $\delta>0$ small, $q>0$ because $\tau_I<\tau_c$ (Proposition~\ref{prop:closure}), and the arguments of the sines lie in $(0,\pi)$ because $\tau_I<\pi/2$; hence $G_P<0$ there except at $\tau_P$. Since $G(\tau;\x)-G_P(\tau)=(x-x_P)\cos\theta-(y-y_P)\sin\theta\to0$ uniformly, for $\x\in\Eset\cap\{x\ge0\}$ near $P$ all roots of $G(\cdot;\x)$ in $[0,\tau_I+\delta]$ are close to $\tau_P$: if $\x_n\to P$ and $\tau_n$ are such roots, every limit point $\bar\tau$ of $(\tau_n)$ satisfies $G_P(\bar\tau)=0$ by uniform convergence, so $\bar\tau=\tau_P$; for such $\x$, $W(\x)<\tau_I$ is the right-branch root (Proposition~\ref{prop:Wform}), so $W(\x)\to\tau_P$. If $P\in\Bbar_+\setminus\{c\}$, all points of $\Eset$ close to $P$ have $x>0$. At $P=c$, the reflected point $\x^\dagger=(|x|,y)$ satisfies $\x^\dagger\in\Eset$, $\x^\dagger\to c$ and $W(\x)=W(\x^\dagger)$ by symmetry, so $W(\x)\to\tau_I$ also for approaches from $x<0$. The statement for $\Bbar_-$ follows by reflection.
\end{proof}

\begin{lemma}[Invariance under the relay]\label{lem:inv}
Let $\x$ be a relay solution starting in $K$, and consider times at which it is still in $K$. (a) $\x(t)\notin\NUP$ for $t>0$. (b) $\Hset$ and $\Hset_*$ are positively invariant. (c) The set $\mathcal T$ of times with $\x(t)\in\Bbar$ is an interval. For almost every $t\in\mathcal T$, either $\x(t)=c$ and $\dot\x(t)=0$, or $\x(t)\ne c$, the evader plays $\mathbf u(t)=-\nu\hat{\boldsymbol\lambda}$, with $\hat{\boldsymbol\lambda}$ the unit normal of the barrier pointing into $\Hset$, and $\frac{d}{dt}\tau_{\x(t)}=-1$.
\end{lemma}
\begin{proof}
(a) Near $\NUP_+$, $x>0$ and the relay plays $(-1,-1)$. With $\ell=x\cos\alpha-y\sin\alpha$,~\eqref{eq:saddle_ray} gives $\dot\ell\le\rho_B-\rho<0$ on $\NUP_+$, hence $\dot\ell<0$ on a neighbourhood, and $\ell$ cannot rise to $0$ there from inside $K$.

(b) Near each $P\in\partial\Hset$ we use a function $D$ with $\{D\ge0\}=\Hset$ and $\{D>0\}=\operatorname{int}\Hset$ locally and
\begin{equation}\label{eq:gron}
\tfrac{d}{dt}\,D(\x(t))\ge-L\,|D(\x(t))|\quad\text{for a.e. $t$ with $\x(t)$ near $P$.}
\end{equation}
If $P\ne c$, say $x_P>0$, the relay plays the constant control $\mathbf a_+=(-1,-1)$ near $P$, and $\partial\Hset$ is a $C^{1,1}$ curve there: barrier arcs are analytic, the $\NUP$ is straight, and the two join tangentially at $B_+$ by~\eqref{eq:dpp}. Let $D$ be the signed distance to $\partial\Hset$, positive in $\operatorname{int}\Hset$. The function $m(\x)=\min_{\|\mathbf u\|\le\nu}\nabla D(\x)^\top\dot\x(\x,\mathbf u,\mathbf a_+)$ is Lipschitz, vanishes on the barrier by semipermeability, and is at least $\rho-\rho_B\ge0$ on the $\NUP$; hence, if $\hat\x$ is a nearest point of $\partial\Hset$ and $L$ bounds the Lipschitz constant of $m$, $m(\x)\ge m(\hat\x)-L\|\x-\hat\x\|\ge-L|D(\x)|$. Since $\tfrac{d}{dt}D(\x(t))=\nabla D^\top\dot\x\ge m(\x(t))$, this gives~\eqref{eq:gron}. If $P=c$, let $D_R$ be the signed distance to the $\BUP$ characteristic of $B_+$ continued slightly past $c$, positive on the side of $\Hset$, and let $D=D_R$ on $\{x\ge0\}$ and $D(x,y)=D_R(-x,y)$ on $\{x\le0\}$. Where $x(t)\ne0$ the previous argument applies. On $\{t:x(t)=0\}$, $\dot x=0$ a.e., and the derivative of $D(\x(t))$ is $\partial_yD_R(0,y)\,\dot y\ge0$, because $\dot y=u_y+1\ge1-\nu$ and $\nabla D_R(c)=(-\cos\theta_I,\sin\theta_I)^\top$. Write $D(t)=D(\x(t))$ while the state remains in the neighbourhood of $P$. From~\eqref{eq:gron}, $D\ge0$ is preserved: if $D(t_0)\ge0$ and $D(t_1)<0$, let $t_2$ be the last zero of $D$ before $t_1$; on $(t_2,t_1]$, $D<0$ and $(D(t)e^{-Lt})'\ge0$, so $D(t_1)e^{-Lt_1}\ge D(t_2)e^{-Lt_2}=0$, a contradiction. Similarly, $D>0$ is preserved, since $D(t)\ge D(t_0)e^{-L(t-t_0)}>0$ for $t\ge t_0$ while $D>0$. These local estimates are patched together by a first-crossing argument: a first time at which the state leaves $\Hset$, or its ambient interior, would lie in one of the neighbourhoods above, where the estimates exclude it. A trajectory starting on $\NUP$ enters the ambient interior immediately, by the strict inward inequality in (a), and cannot return to $\NUP$. Since $\Hset_*=(\operatorname{int}\Hset)\cup\NUP$, this also proves invariance of $\Hset_*$, as required in (b).

(c) By (b), $\Hset$ is entered from $\Eset$ at most once, through $\Bbar$, and $\Hset_*$ is never left, so $\mathcal T$ is an interval. At almost every $t\in\mathcal T$, $\frac{d}{dt}D(\x(t))=0$. If $\x(t)\in\Bbar_+\setminus\{c\}$, this derivative is $\hat{\boldsymbol\lambda}^\top\dot\x\ge m=0$, with equality only for the unique minimizer $\mathbf u=-\nu\hat{\boldsymbol\lambda}$; then $\dot\x$ is the saddle velocity of the $\BUP$ characteristic, which is regular because $\tau_I<\tau_c$, and $\tau_{\x}$ decreases at unit rate. If $\x(t)=c$ on a set of positive measure, $\dot\x=0$ a.e. there; for $\nu=1$ this occurs with $\mathbf u=(0,-1)^\top$ (Remark~\ref{rem:nu1}).
\end{proof}

\begin{lemma}[$W_c$ along relay solutions]\label{lem:mono}
Along every relay solution starting in $K$, $t\mapsto W_c(\x(t))+t$ is nondecreasing while $\x(t)\in K$.
\end{lemma}
\begin{proof}
The sets $\Eset$, $\Bbar$ and $\Hset_*$ partition $K$. We follow the solution through them. Within $\Eset$, $W_c=W$ decreases at most at unit rate; on entry from $\Eset$ into $\Bbar$ there is no downward jump, because $W\to W_c$ by Lemma~\ref{lem:cont}(b); within $\Bbar$, $\tau$ decreases at unit rate or is constant at $c$; entry into $\Hset_*$ is an upward jump to $+\infty$; and by Lemma~\ref{lem:inv} the solution never returns from $\Hset_*$ or from $\Hset$. In detail: if $\x(0)\in\Hset_*$, Lemma~\ref{lem:inv}(b) keeps $\x(t)$ in $\Hset_*$, where $W_c=+\infty$, and the claim holds. If $\x(0)\in\Bbar$, let $\mathcal T$ be the barrier interval of Lemma~\ref{lem:inv}(c), which starts at $0$. The map $P\mapsto\tau_P$ is Lipschitz on each barrier arc, which is regular because $\tau_I<\tau_c$, and the two arcs meet at $c$ with the common value $\tau_I$; hence $t\mapsto\tau_{\x(t)}$ is absolutely continuous on compact subintervals of $\mathcal T$. By Lemma~\ref{lem:inv}(c), its derivative is $-1$ a.e. where $\x(t)\neq c$ and $0$ a.e. where $\x(t)=c$, so $W_c(\x(t))+t$ is nondecreasing on $\mathcal T$; after $\mathcal T$ the solution is in $\Hset_*$, where $W_c=+\infty$. If $\x(0)\in\Eset$, the solution stays in $\Eset$ until its first entry $t_1$ into $\Hset$, if any, which occurs through $\Bbar$ (Lemma~\ref{lem:inv}(b)). On $[0,t_1)$, $W\circ\x$ is locally absolutely continuous (Lemma~\ref{lem:cont}(a)) and, by the pursuer part of the proof of Theorem~\ref{thm:slow}, $\frac{d}{dt}W(\x)\ge-1$ a.e. If $t_1$ is finite, $\x(t_1)\in\Bbar$ and $W(\x(t))\to W_c(\x(t_1))$ as $t\uparrow t_1$ by Lemma~\ref{lem:cont}(b), and from $t_1$ on the barrier case applies. If the solution leaves $K$ without entering $\Hset$, the first part covers it.
\end{proof}

\begin{lemma}[Sampled relay]\label{lem:samp}
Let $\x_0\in K$ and $T<W_\sharp(\x_0)$, and let the pursuer play $v_p=-1$, $\omega_p=-\sgn x(kh)$ on $[kh,(k+1)h)$, with $\sgn0:=1$. There is $h_0>0$ such that, for every $h\in(0,h_0)$ and every measurable evader control, $\x(t)\in K$ on $[0,T]$.
\end{lemma}
\begin{proof}
\emph{Contradiction sequence.} Otherwise there are $h_n\downarrow0$, evader controls $\mathbf u_n$ and solutions $\x_n$ of the sampled closed loop that leave $K$ at times $t_n\le T$, with $\x_n(t_n)\in\partial K$. Extend~\eqref{eq:reduced} to the whole plane. Since $\|\dot\x\|\le\nu+1+\|\x\|$, Gr\"onwall's inequality makes the $\x_n$ uniformly bounded and equi-Lipschitz on $[0,T+1]$. Along a subsequence, $\x_n\to\bar\x$ uniformly on $[0,T+1]$, $t_n\to t^*\le T$, and $\mathbf u_n$, $\omega_n$ converge weakly-$*$ in $L^\infty$ to $\bar{\mathbf u}$, $\bar\omega$, with $\|\bar{\mathbf u}\|\le\nu$ and $|\bar\omega|\le1$ because these sets are closed and convex.

\emph{Identification of the limit.} With $v_p=-1$, the dynamics $\dot x=u_x+\omega y$, $\dot y=u_y+1-\omega x$ contain no product of $\mathbf u$ and $\omega$, so their integral form passes to the limit; for instance,
\[
\int_0^t(\omega_ny_n-\bar\omega\bar y)\,ds=\int_0^t\omega_n(y_n-\bar y)\,ds+\int_0^t(\omega_n-\bar\omega)\bar y\,ds\to0,
\]
the first term by uniform convergence and $|\omega_n|\le1$, the second by weak-$*$ convergence. If $\bar x>0$ on an interval $[a,b]$, then $x_n>0$ on $[a,b]$ for large $n$, in particular at the sampling instants in $[a,b]$, so $\omega_n=-1$ on $[a+h_n,b]$; hence $\bar\omega=-1$ almost everywhere where $\bar x>0$, and $\bar\omega=1$ where $\bar x<0$. Together with $|\bar\omega|\le1$ where $\bar x=0$, this shows that $\bar\x$ solves the Filippov inclusion of Lemma~\ref{lem:relay} with control $\bar{\mathbf u}$, so it is the relay solution for $\bar{\mathbf u}$. Moreover $\bar\x\in K$ on $[0,t^*]$ and $\bar\x(t^*)\in\partial K$.

\emph{Boundary contact.} Uniform convergence gives only contact of $\bar\x$ with $\partial K$, not a strict exit. Contact with the $\NUP$ is excluded as the limit of the exits: if $\bar\x(t^*)\in\NUP_+$, then near it $x>0$, the sampled control is $(-1,-1)$ for large $n$ on a time neighbourhood of $t^*$, and $\dot\ell\le\rho_B-\rho+o(1)<0$ there by~\eqref{eq:saddle_ray}, so $\ell(\x_n)\le0$ is preserved and $\x_n$ cannot leave $K$ near $\bar\x(t^*)$, contradicting $t_n\to t^*$. Hence $\bar\x(t^*)\in\UP\cup\{\mathbf 0,B_+,B_-\}$, where $W_c=0$: on the $\UP$ and at the apex because $W=0$ there, and at $B_\pm$ because $\tau_{B_\pm}=0$.

\emph{General initial states.} By Lemma~\ref{lem:mono}, $W_c(\bar\x(t^*))+t^*\ge W_c(\x_0)$, so $t^*\ge W_c(\x_0)$. If $\x_0\notin\Bbar\setminus\{c\}$, then $W_c(\x_0)=W_\sharp(\x_0)>T\ge t^*$, a contradiction.

\emph{Barrier states off the corner.} Let $\x_0\in\Bbar_+\setminus\{c\}$. Here $W_c(\x_0)=\tau_{\x_0}$ is finite while $W_\sharp(\x_0)=+\infty$, so Lemma~\ref{lem:mono} is not enough. Since $\Hset$ and $\Hset_*$ are positively invariant under the relay (Lemma~\ref{lem:inv}(b)) and $\bar\x(t^*)\notin\Hset_*$, $\bar\x$ stays on $\Bbar$ on $[0,t^*]$, and $\bar\x(t^*)\in\Hset\cap\partial K$ is $B_+$. By Lemma~\ref{lem:inv}(c) it moves along $\Bbar_+$ with $\tau$ decreasing at unit rate, so on $[0,t^*]$ it stays on the arc of $\Bbar_+$ from $\x_0$ to $B_+$, which lies in $\{x\ge x_*\}$ for some $x_*>0$ because $\x_0\ne c$. For large $n$, $\x_n$ stays in $\{x>x_*/2\}$ on $[0,\bar t\,]$ for some $\bar t>t^*$. There every sample has $x>0$, so the sampled control is exactly the constant $(-1,-1)$, and the signed-distance estimate of Lemma~\ref{lem:inv}(b) for $P\ne c$ keeps $\x_n$ in $\Hset\subset K$ on $[0,\bar t\,]$, contradicting $t_n\to t^*<\bar t$.

\emph{The corner.} At $c$ the lower bound on $|x|$ is unavailable, since the sampled relay may switch sign along trajectories near the axis, and the argument gives only horizons $T<W_\sharp(c)=\tau_I$. Lemma~\ref{lem:corner} shows that this bound is sharp.
\end{proof}

\begin{lemma}[Sampled evader feedback]\label{lem:evs}
Let $\x_0\in\Eset$ and $\varepsilon>0$, and let the evader hold on $[kh,(k+1)h)$ the feedback of Theorem~\ref{thm:slow} evaluated at $\x(kh)$, using the right branch if $x(kh)\ge0$. There is $h_0>0$ such that, for $h<h_0$ and every measurable pursuer control, $t_f\le W(\x_0)+\varepsilon$.
\end{lemma}
\begin{proof}
By symmetry let $x_0\ge0$.

\emph{Chart.} Choose $\rho_f(\x_0)<\bar\rho<\rho_B$, and let $\mathcal C$ be the image of the compact characteristic strip
\[
0\le\rho_f\le\bar\rho,\qquad 0\le\tau\le\tau_D(\rho_f),
\]
where $\tau_D$ is the first axis-contact time and $\tau_D(0)=0$. Thus $\mathcal C$ includes its usable-ray and axis edges but is separated from the barrier. By Proposition~\ref{prop:Wform}, $G_\tau=\rho_B-\rho_f\ge\rho_B-\bar\rho>0$ on $\mathcal C$, so the implicit function theorem extends $W_R$, $\rho_f$ and the right-branch control $\mathbf k_E$ smoothly to a fixed neighbourhood of $\mathcal C$, across the axis and the usable ray. Choose a compact neighbourhood $Q$ of $\mathcal C$ within this chart, and let $L$ be a Lipschitz constant of $\mathbf k_E$ and $M$ a bound on $\|\dot\x\|$ on $Q$, uniform in the controls.

\emph{Held-feedback error.} While the state stays in $Q$, $\|\x(t)-\x(kh)\|\le Mh$ on each sampling period, so $\|\mathbf k_E(\x(kh))-\mathbf k_E(\x(t))\|\le LMh$. The held control therefore differs from the feedback at the current state by at most $LMh$, and~\eqref{eq:Wdot}--\eqref{eq:rhodot}, which hold for the exact feedback and every pursuer control on the part of $Q$ with $x\ge0$ in $K$, give $\dot W_R\le-1+C_0h$ and $\dot\rho_f\le C_0h$ for a constant $C_0$, uniformly over pursuer controls.

\emph{Localization.} Let $T_0=W(\x_0)+1$, and choose $h$ so that $C_0hT_0<\bar\rho-\rho_f(\x_0)$ and $W(\x_0)/(1-C_0h)<T_0$. At the axis, the margin $\dot x\ge\nu C-y>0$ of Theorem~\ref{thm:slow} is uniformly positive on $\mathcal C$, and by continuity it remains positive, held error included, in a strip $\{0\le x\le\varepsilon_1\}\cap Q$ for small $h$; hence the state enters $\{x>0\}$ immediately if it starts on the axis and cannot cross back. Up to time $T_0$, $\rho_f$ increases by less than $\bar\rho-\rho_f(\x_0)$, so the state cannot reach the edge $\rho_f=\bar\rho$. A first-exit argument now applies: the only remaining edges of $\mathcal C$ are the usable ray and the apex, so the state stays in $\mathcal C$ until it reaches them. This localization is established before the decrease estimate is integrated.

\emph{Escape time.} Since $W_R\ge0$ in $\mathcal C$ and $\dot W_R\le-1+C_0h$, the state reaches the usable ray or the apex by $W(\x_0)/(1-C_0h)<T_0$. There the outward normal velocity has a positive margin, uniform for $\rho_f\le\bar\rho$ and for small held errors, so strict exit follows immediately, also when $W(\x_0)=0$. Choosing $h$ with $W(\x_0)/(1-C_0h)\le W(\x_0)+\varepsilon$ gives $t_f\le W(\x_0)+\varepsilon$. The threshold $h_0$ depends on $\x_0$ and $\varepsilon$ through $\bar\rho$; the proof gives no estimate uniform as $\x_0$ approaches the barrier.
\end{proof}

\begin{lemma}[Evader at the corner]\label{lem:corner}
There is a nonanticipative evader strategy $A^c\in\mathcal S_E$ such that, from $c$, $J(A^c[\mathbf a],\mathbf a)<\tau_I$ for every measurable pursuer control $\mathbf a$. Hence $V_E(c)\le\tau_I$.
\end{lemma}
\begin{proof}
For $0\le t\le\tau_I$ let $\theta(t)=\theta_I-t$, $S=\sin\theta(t)$, $C=\cos\theta(t)$, and let $\mathbf u_R(t)=\nu(C,-S)^\top$ and $\mathbf u_L(t)=\nu(-C,-S)^\top$. These are the evader's controls along the $\BUP$ characteristics traversed forward from $c$. Let $\mathbf a_R=(-1,-1)$ and $\mathbf a_L=(-1,1)$. The idea is the following: each escape side has exactly one exceptional pursuer response, the control that saturates that branch, and the evader chooses from the pursuer's initial control a side whose exceptional response is excluded.

\emph{Certificate.} Under $\mathbf u_R$ and a pursuer control $\mathbf a=(v,\omega)$, let $L_R(t)=G(\tau_I-t;\x(t))=(x+1)C-yS-\cos\alpha+\nu(\tau_I-t)$, with $G$ of~\eqref{eq:Gscalar}. Then $L_R(0)=0$, because $c$ lies on the $\BUP$ characteristic at retro-time $\tau_I$, and $\dot C=S$, $\dot S=-C$ and~\eqref{eq:reduced} give
\[
\dot L_R=(v+1)S+(\omega+1)\psi,\qquad \psi=xS+yC .
\]
Both coefficients are nonnegative on the bow tie, and $S>0$ because $\alpha\le\theta(t)\le\theta_I<\pi/2$.

\emph{The state stays in $x>0$.} Let $y_D(\tau)=(\cos(\alpha+\tau)-\cos\alpha+\nu\tau)/\sin(\alpha+\tau)$ be the axis-contact height of Proposition~\ref{prop:Wform}(c), so that $G(\tau;(0,y_D(\tau)))=0$. Then $y_D'(\tau)=\Gamma(\alpha+\tau)/\sin^2(\alpha+\tau)>0$ on $[0,\tau_I)$, with $\Gamma$ of~\eqref{eq:FI}, so $0\le y_D\le y_D(\tau_I)=y_I$. While the state is in $K$ with $x\ge0$, $\psi\ge0$, so $L_R\ge0$. At a point of the axis, $L_R\ge0$ gives $y\le y_D(\tau_I-t)\le y_I$, hence
\[
\dot x=\nu C+\omega y\ge\nu\cos\theta_I-y_I=\frac{\sin\theta_I(1-\nu\sin\theta_I)}{\cos\theta_I}=:d_0>0
\]
for every admissible control. This bound has a continuous right side, so $\dot x\ge d_0/2$ a.e. near any axis contact in $K$: the state enters $x>0$ immediately and cannot return to the axis while it is in $K$. Thus $\psi>0$ for $t>0$ until exit.

\emph{Exit.} Suppose the state stays in $K$ on $[0,\tau_I]$. Then $L_R(\tau_I)=\int_0^{\tau_I}\bigl[(v+1)S+(\omega+1)\psi\bigr]\,dt$, which is positive unless $\mathbf a=\mathbf a_R$ almost everywhere on $(0,\tau_I)$. But $L_R(\tau_I)=x\cos\alpha-y\sin\alpha\le0$ in $K$. Hence, if $\mathbf a\ne\mathbf a_R$ on a set of positive measure in $(0,\tau_I)$, the state leaves $K$ strictly before $\tau_I$. Symmetrically, under $\mathbf u_L$ the certificate $L_L=(1-x)C-yS-\cos\alpha+\nu(\tau_I-t)$ satisfies $\dot L_L=(v+1)S+(1-\omega)(yC-xS)$, and the state leaves $K$ before $\tau_I$ whenever $\mathbf a\ne\mathbf a_L$ on a set of positive measure in $(0,\tau_I)$. If $\mathbf a=\mathbf a_R$ a.e., $L_R\equiv0$ and the state follows $\Bbar_+$ to $B_+$, which is not a strict exit.

\emph{Selection.} Let $\mathcal G_R$ be the set of pursuer controls with $\mathbf a=\mathbf a_R$ a.e. on some interval $(0,\varepsilon)$, $\varepsilon>0$, and let $A^c[\mathbf a]=\mathbf u_L$ on $(0,\tau_I)$ if $\mathbf a\in\mathcal G_R$, $A^c[\mathbf a]=\mathbf u_R$ there otherwise, and $A^c[\mathbf a]=0$ after $\tau_I$. If $\mathbf a$ and $\mathbf b$ coincide a.e. on $[0,s]$ with $s>0$, then $\mathbf a\in\mathcal G_R$ if and only if $\mathbf b\in\mathcal G_R$, so the outputs coincide on $[0,s]$ and $A^c$ is nonanticipative; $\mathcal G_R=\bigcup_n\{\mathbf a:\int_0^{\tau_I/n}(|v+1|+|\omega+1|)\,dt=0\}$ is a Borel set. If $\mathbf a\in\mathcal G_R$, then $\mathbf a=\mathbf a_R\ne\mathbf a_L$ on an initial interval and $\mathbf u_L$ forces exit before $\tau_I$. If $\mathbf a\notin\mathcal G_R$, then $\mathbf a$ is not a.e.\ equal to $\mathbf a_R$ on $(0,\tau_I)$ and $\mathbf u_R$ forces exit before $\tau_I$.
\end{proof}

\begin{proof}[Proof of Theorem~\ref{thm:value}]
By Corollary~\ref{cor:sandwich}, $V_E\le W\le V_P$, and $V_P=W=+\infty$ on $\Hset$.

\emph{$V_P\le W$ on $\Eset$.} Let $\x_0\in\Eset$, $\varepsilon>0$ and $B\in\mathcal S_P$. Pair $B$ with the sampled evader feedback of Lemma~\ref{lem:evs} interval by interval, as described at the beginning of this appendix: on $[kh,(k+1)h)$ the evader's control is the constant fixed by $\x(kh)$, and nonanticipativity of $B$ fixes the pursuer's control on that interval and hence $\x((k+1)h)$. This produces an evader control $\mathbf u$ with $J(\mathbf u,B[\mathbf u])\le W(\x_0)+\varepsilon$ by Lemma~\ref{lem:evs}. Thus
\[
\forall B\in\mathcal S_P\ \exists\mathbf u:\ J(\mathbf u,B[\mathbf u])\le W(\x_0)+\varepsilon\quad\Longrightarrow\quad V_P(\x_0)\le W(\x_0)+\varepsilon,
\]
and $\varepsilon\downarrow0$ gives $V_P\le W$.

\emph{$V_E\ge W_\sharp$ on $K$.} Let $\x_0\in K$, $T<W_\sharp(\x_0)$ and $A\in\mathcal S_E$. Pairing $A$ with the sampled relay of Lemma~\ref{lem:samp} in the same way produces a pursuer control $\mathbf a$ with $J(A[\mathbf a],\mathbf a)\ge T$. Thus
\[
\forall A\in\mathcal S_E\ \exists\mathbf a:\ J(A[\mathbf a],\mathbf a)\ge T\quad\Longrightarrow\quad V_E(\x_0)\ge T,
\]
and letting $T\uparrow W_\sharp(\x_0)$ gives $V_E\ge W_\sharp$; where $W_\sharp(\x_0)=0$ this is immediate, the payoff being nonnegative. When $W_\sharp(\x_0)=+\infty$, this provides a pursuer control for every finite horizon $T$, not necessarily one control that works for all horizons at once; that suffices for $V_E(\x_0)=+\infty$. Since $W_\sharp=W$ off $c$ and $W_\sharp(c)=\tau_I$, this gives $V_E=W$ on $K\setminus\{c\}$ and $V_E(c)\ge\tau_I$. Lemma~\ref{lem:corner} gives $V_E(c)\le\tau_I$, and $V_P(c)=+\infty$ because $c\in\Hset$. Lemma~\ref{lem:corner} gives a strict bound for each pursuer control, but by $V_E(c)\ge\tau_I$ the supremum of these payoffs over pursuer controls is $\tau_I$.
\end{proof}

\end{document}